\documentclass{article} 
\usepackage{iclr2026_conference,times}

\usepackage[utf8]{inputenc} 
\usepackage[T1]{fontenc}    
\usepackage{hyperref}       
\usepackage{url}            
\usepackage{booktabs}       
\usepackage{amsfonts}       
\usepackage{nicefrac}       
\usepackage{microtype}      
\usepackage{xcolor}         
\usepackage{colortbl}
\usepackage{mathtools}

\usepackage{graphicx}
\usepackage{xcolor}
\usepackage{xspace}
\usepackage{amsmath}
\usepackage{amssymb}
\usepackage{bbm} 
\usepackage{subcaption}
\usepackage{algorithm}
\usepackage{algpseudocode}
\usepackage{mdwlist}
\usepackage{multirow} 
\usepackage{cuted}
\usepackage{afterpage}
\usepackage{stfloats}
\usepackage{hhline}
\usepackage{tabularx}
\usepackage{kotex}
\usepackage{wrapfig}
\usepackage{cellspace}
\usepackage{amsthm}
\usepackage{enumitem}

\usepackage{makecell}
\usepackage{tcolorbox}

\usepackage{amsmath,amsfonts,bm}

\def\eqref#1{equation~\ref{#1}}

\def\1{\bm{1}}

\DeclareMathAlphabet{\mathsfit}{\encodingdefault}{\sfdefault}{m}{sl}
\SetMathAlphabet{\mathsfit}{bold}{\encodingdefault}{\sfdefault}{bx}{n}

\DeclareMathOperator*{\argmax}{arg\,max}
\DeclareMathOperator*{\argmin}{arg\,min}

\newcommand{\hide}[1]{}

\definecolor{ForestGreen}{rgb}{0.133, 0.545, 0.133}
\definecolor{torange}{rgb}{0.949, 0.522, 0.}
\definecolor{lightlavender}{RGB}{240, 240, 255}
\definecolor{lightgrey}{rgb}{0.98, 0.98, 0.98}

\usepackage{pifont}

\newcommand{\githublink}{\url{https://github.com/snudatalab/PQQP}\xspace}

\newtheorem{problem}{Problem}
\newtheorem{hypothesis}{Hypothesis}
\newtheorem{definition}{Definition}

\newtheorem{theorem}{Theorem}

\newtheorem{assumption}{Assumption}

\newcolumntype{Y}{>{\centering\arraybackslash}X}
\newcommand{\smallsection}[1]{\noindent\smash{\textbf{#1.}}}
\setlist[itemize]{topsep=1pt, itemsep=0pt}

\newcommand{\mat}[1]{\mathbf{#1}}

\newcommand{\pg}[1]{\mathcal{G}(#1)}
\newcommand{\coa}[1]{\mathcal{A}\big(#1\big)}

\newcommand{\gaussian}[1]{\mathcal{N}(#1)}

\newcommand{\model}{\phi}
\newcommand{\compmodel}{\phi'}
\newcommand{\comprate}{C}
\newcommand{\comprateprune}{C_{\mathcal{P}}}
\newcommand{\compratequant}{C_{\mathcal{Q}}}
\newcommand{\compeqrate}{C^{*}}
\newcommand{\compeqrateprune}{C^{*}_{\mathcal{P}}}

\newcommand{\pruneratio}{p}
\newcommand{\orgbits}{B_{orig}}
\newcommand{\quantbits}{B_{\mathcal{Q}}}
\newcommand{\compeqratepruneone}{C^{*}_{\mathcal{P}_1}}

\newcommand{\compratepruneone}{C_{\mathcal{P}_1}}
\newcommand{\compratequantone}{C_{\mathcal{Q}_1}}

\newcommand{\comprateprunetwo}{C_{\mathcal{P}_2}}
\newcommand{\compratequanttwo}{C_{\mathcal{Q}_2}}

\newcommand{\metriconly}{\mathcal{M}}
\newcommand{\metric}[1]{\mathcal{M}(#1)}

\newcommand{\layeri}{l_i}
\newcommand{\layerj}{l_j}
\newcommand{\layers}{\mathbb{L}}

\newcommand{\weighti}{\mat{W}_i}
\newcommand{\acti}{\mat{X}_i}
\newcommand{\weightj}{\mat{W}_j}
\newcommand{\actj}{\mat{X}_j}

\newcommand{\ptq}{\mathcal{Q} \circ \mathcal{P}}
\newcommand{\qtp}{\mathcal{P} \circ \mathcal{Q}}

\newcommand{\unit}{u}
\newcommand{\uniti}{u_{i}}
\newcommand{\units}[1]{\mathbb{U}(#1)}
\newcommand{\applies}[3]{\mathbb{D}_{#1}^{#2}(#3)}

\newcommand{\error}[1]{\delta(#1)}
\newcommand{\erroru}[2]{\delta_{#1}(#2)}
\newcommand{\errorp}[1]{\delta_{\mathcal{P}}(#1)}
\newcommand{\errorq}[1]{\delta_{\mathcal{Q}}(#1)}
\newcommand{\errorptq}[1]{\delta_{\ptq}(#1)}
\newcommand{\errorqtp}[1]{\delta_{\qtp}(#1)}

\newcommand{\compset}{\mathbb{F}}
\newcommand{\compone}[1]{f_{1}(#1)}
\newcommand{\componeonly}{f_{1}}
\newcommand{\comptwo}[1]{f_{2}(#1)}
\newcommand{\comptwoonly}{f_{2}}
\newcommand{\compn}[1]{f_{n}(#1)}

\newcommand{\comp}[1]{f(#1)}
\newcommand{\componly}{f}

\newcommand{\permopt}{\pi^{*}}
\newcommand{\perm}{\pi}
\newcommand{\permset}{\Pi}

\newcommand{\typeset}[1]{\mathcal{T}_{#1}}
\newcommand{\type}{t}
\newcommand{\typegran}[1]{t_{#1}}
\newcommand{\lut}{\type_{\text{lut}}}

\newcommand{\quant}[1]{\mathcal{Q}(#1)}
\newcommand{\prune}[1]{\mathcal{P}(#1)}
\newcommand{\quantonly}{\mathcal{Q}}
\newcommand{\pruneonly}{\mathcal{P}}
\newcommand{\prunedlayers}{\mathbb{P}}
\newcommand{\prunei}[1]{\mathcal{P}_{i}(#1)}
\newcommand{\allprunes}[1]{\mathbb{P}(#1)}

\newcommand{\groupone}{\mathbb{G}_1}
\newcommand{\grouptwo}{\mathbb{G}_2}
\newcommand{\groupthree}{\mathbb{G}_3}
\newcommand{\groupfour}{\mathbb{G}_4}

\newcommand{\interference}[1]{\Delta(#1)}
\newcommand{\interferenceonly}{\Delta}

\usepackage{hyperref}
\usepackage{url}

\graphicspath{ {images/} }

\title{Prune-then-Quantize or Quantize-then-Prune? Understanding the Impact of Compression \\ Order in Joint Model Compression}

\author{Minjun Kim, Jaehyeon Choi, Hyunwoo Yang, Jongjin Kim, Jinho Song \& U Kang\thanks{Corresponding Author.} \\
Seoul National University, Seoul, South Korea \\
\texttt{\{minjun.kim,ukang\}@snu.ac.kr} \\
}

 \iclrfinalcopy 

\begin{document}

\maketitle

\begin{abstract}
    What happens when multiple compression methods are combined—does the order in which they are applied matter?
Joint model compression has emerged as a powerful strategy to achieve higher efficiency by combining multiple methods such as pruning and quantization.
A central but underexplored factor in joint model compression is the compression order, or the sequence of different methods within the compression pipeline.
Most prior studies have sidestepped the issue by assuming orthogonality between techniques, while a few have examined them only in highly constrained cases.
Consequently, the broader role of compression order in shaping model performance remains poorly understood.
In this paper, we address the overlooked problem of compression order and provide both theoretical and empirical analysis.
We formulate the problem of optimizing the compression order and introduce the Progressive Intensity Hypothesis, which states that weaker perturbations should precede stronger ones.
We provide theoretical guarantees showing that the relative benefit of one order increases with the underlying performance gap.
Extensive experiments on both language and vision models validate the hypothesis, and further show its generality to broader setups such as multi-stage compression and mixed-precision quantization. 
\end{abstract}

\section{Introduction}
\label{sec:intro}

\textit{When combining pruning and quantization, which order leads to better model performance?}
Although deep neural networks have achieved remarkable success across diverse domains, deploying them on edge devices remains challenging due to limited computational resources.
To bridge this gap, network compression techniques~\citep{SongHanSurvey, PQSurvey, MCforLLMSurvey, ZSQSurvey} have been proposed, including pruning~\citep{KPrune, SLEB, SPRINT}, quantization~\citep{SensiMix, QuaRot, SynQ}, knowledge distillation~\citep{MustaD, Pea-KD, PET}, parameter sharing~\citep{Roast, BasisSharing} and low-rank approximation~\citep{Falcon, SVDQuant, SVD-LLM}.
Recent studies highlight that combining these compression methods—known as \textit{joint model compression}—achieves better trade-offs between compression ratio and model performance than applying them separately~\citep{PsandQs, DeepComp, AdaptiveQandP}.


A critical yet underexplored issue in joint model compression is the \textit{compression order}—the sequence in which individual compression methods are applied to the target model.
As most of these techniques are not simultaneously applicable and should be executed sequentially~\citep{APQ, PruneVSQuant}, identifying an optimal order can yield a ``free lunch'' by improving performance without any additional computation.
Empirical findings~\citep{AccelFPGAs, OPQ, Automatic} show that the performance of the compressed model is sensitive to the compression order, necessitating a deeper understanding of when and why certain orders work better.

However, the role of compression order has been largely overlooked by prior studies~\citep{oBERT, SmoothQuant, Dejavu}.
Most existing studies implicitly assume that compression order has no effect on the grounds of orthogonality, na\"ively arguing that different techniques operate independently~\citep{PQK, InferenceSurvey, SLEB, JointJournal}.
Only a few works have examined the problem, and most of them merely offer empirical evidence confined to specific settings~\citep{APQ, UnderstandInt4, BoostViT}.
A notable attempt~\citep{Interplay} presents a theoretical framework, proving the non-orthogonality of pruning and quantization, concluding that pruning followed by quantization is always preferable.
However, the scope of the work remains narrow and less practical, focusing only on magnitude-based pruning and max-scaled quantization (see Appendix~\ref{app:subsec:direct_comparison_interplay}).
To date, no study has systematically investigated the tendencies of compression order in general settings, neither empirically nor theoretically.

In this paper, we demonstrate that applying more aggressive compression algorithms at later stages yields superior performance.
We first formulate the problem of \textit{joint compression order optimization} (see Section~\ref{subsec:problem_definition} and Problem~\ref{problem}), and introduce \textit{the Progressive Intensity Hypothesis}, which posits that ordering compression methods from weaker to stronger improves performance (see Hypothesis~\ref{hypothesis}).
Figure~\ref{fig:hypothesis} offers a conceptual depiction of the proposed hypothesis.
We validate our claim through both theoretical analysis and extensive experiments.
Theoretically, we show that the advantage of the compression order grows monotonically with the performance gap between two methods under \textit{disjoint sensitivity} (see Theorem~\ref{thm:1} and Definition~\ref{def:disjoint_selectivity}).
In other cases, we define \textit{interference} as an additional error from mutual interaction and investigate its influence (see Definition~\ref{def:interference}).
Experimentally, we validate the hypothesis across both language and vision models, covering diverse model architectures, tasks, and compression scenarios (see Sections~\ref{subsec:experiment_llms} and~\ref{subsec:experiment_vision}).
Our analysis also considers how factors such as weight-update strategies and rotations affect the role of compression order (see Figures~\ref{fig:llm_result_update_and_rotation} and~\ref{fig:rotation_pruning}).
Moreover, our results highlight that the hypothesis generalizes to broader paradigms, including multi-stage approaches and mixed-precision quantization (see Section~\ref{subsec:beyond}).


Our contributions are summarized as follows:
\begin{itemize}[leftmargin=3.4mm, itemsep=-1mm, topsep=0mm]
    \item \textbf{Formulation.}
    We formally define the novel problem of optimizing the compression order in joint model compression (see Problem~\ref{problem}),
    and propose \textit{the Progressive Intensity Hypothesis}, suggesting that stronger perturbations should be applied later to achieve better performance (see Hypothesis~\ref{hypothesis}).
    \item \textbf{Theory.}
    We provide a theoretical analysis that quantifies the relationship between method interaction and order sensitivity.
	Specifically, we prove that the superiority of one ordering grows monotonically with the performance gap between the two methods (see Theorem~\ref{thm:1}).
    \item \textbf{Experiments.}
    Extensive and consistent experimental results across various domains, models, and tasks support our hypothesis (see Figures~\ref{fig:llm_result},~\ref{fig:llm_result_update_and_rotation}, and~\ref{fig:vision}).
    We further extend the problem to broader setups such as multi-stage compression and mixed-precision quantization (see Figures~\ref{fig:multi_stage} and~\ref{fig:mpq}).
\end{itemize}
To the best of our knowledge, we are the first to both theoretically and experimentally analyze the impact of compression order in joint model compression under general and practical settings.

\smallsection{Reproducibility}
All of our implementation and datasets are available at \githublink.
%

\begin{figure}[t]
	\begin{tcolorbox}[colframe=black!100, colback=cyan!10, boxrule=1pt, arc=5pt, left=5pt, right=5pt, top=5pt, bottom=5pt]
		\textbf{The Progressive Intensity Hypothesis.}
		Neural networks compressed by multiple methods perform better when weaker perturbations are applied first and stronger ones later.
	\end{tcolorbox}
	\vspace{3mm}
	\centering
	\includegraphics[width=\linewidth]{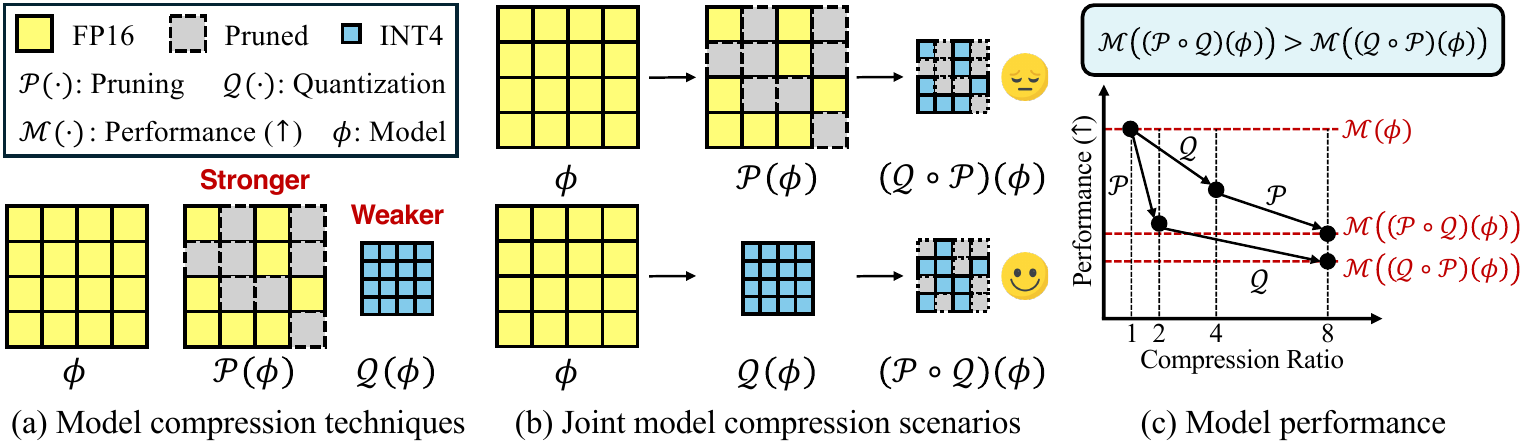}
	\caption{
		\textbf{The Progressive Intensity Hypothesis:}
		Given two compression techniques,
		we conjecture that compressed models perform better if the stronger method is applied after the weaker one.
		That said, the optimal order between pruning and quantization varies with their compression ratios.
	}
	\label{fig:hypothesis}
	\vspace{-4mm}
\end{figure}

\section{Preliminaries and Related Works}
\label{sec:prelim}
We briefly describe the preliminaries and related works
on pruning, quantization, and joint model compression.
The notations used throughout this paper are formally defined in Appendix~\ref{app:sec:notation}.

\smallsection{Pruning and Quantization}
Compression\footnote{In the remainder of the paper, we use `compression' to refer to `model compression' for simplicity.}  techniques aim to transform a pre-trained model $\model$ as a more efficient version $\compmodel$ while minimizing performance degradation~\citep{McAuleySurvey, MLCompSurvey, PPFSurvey}.
This process inevitably introduces an error term $\error{\cdot}$, representing the deviation between outputs of $\compmodel$ and $\model$, which typically increases with the compression ratio $\comprate$.
We define the compression ratio $\comprate$ as the memory usage of $\model$ divided by that of $\compmodel$.
Among various compression techniques $\comp{\model; \comprate}$, our work centers on two major forms: pruning and quantization.

Pruning $\prune{\cdot}$ directly discards less important components of a model to achieve the desired compression ratio while retaining its most critical parts~\citep{KCM, SliceGPT, KPrune}.
Based on the level of granularity, pruning methods fall into three categories:
structured pruning~\citep{SLEB} removes entire structural elements such as layers, filters, or attention heads, semi-structured pruning~\citep{LPViT} enforces fixed sparsity patterns (e.g., 2:4 sparsity) across tensors, and unstructured pruning~\citep{SparseGPT} prunes weights in a fully flexible manner.
In the case of structured pruning at the layer level, the induced error $\errorp{\weighti, \acti}$ is $-\weighti \acti$ when pruning is applied to layer $\layeri$ with weight $\weighti$ and activation $\acti$, and $\mat{0}$ otherwise.
The model achieves a compression ratio $\comprateprune = 1/(1-\pruneratio)$ by pruning a fraction $\pruneratio$ of weights.

Quantization $\quant{\cdot}$ reduces the bit precision used to represent weights and activations by encoding a high-bit network into a lower-bit format~\citep{GholamiSurvey}.
Common quantization techniques include uniform~\citep{BRECQ}, non-uniform~\citep{GanQ}, binary coding~\citep{UniQuanF}, and vector quantization (VQ)~\citep{QuipSharp}.
Although some techniques such as VQ focus only on weight quantization without compressing activations, our main scope is on compressing both for practical acceleration.
A main challenge towards robust quantization is the activation outliers~\citep{SmoothQuant, OWQ}, but recent rotation-based methods~\citep{DuQuant, SpinQuant} have largely overcome it.
The layer-wise error by quantization $\quant{\cdot}$ for a layer $\layeri$ with weight $\weighti$ and activation $\acti$ is computed as $\errorq{\weighti, \acti} = \quant{\weighti} \quant{\acti} - \weighti \acti$, with a compression ratio $\compratequant = \orgbits/\quantbits$ depending on the original $\orgbits$ and target $\quantbits$ bit-widths.

\smallsection{Joint Model Compression}
Joint compression combines two or more compression methods, achieving higher compression ratios while minimizing performance loss~\citep{APQ, UnderstandInt4, BoostViT, Interplay}.
These methods fall into two categories: co-designed and post-hoc frameworks.
Although the former offers the benefit of integration-aware design,
they tend to be method-specific and less adaptable to alternative configurations~\citep{Automatic}.

In contrast, combining independently designed techniques allows for method-agnostic pipelines that adapt easily to diverse architectures.
Several pruning works~\citep{oBERT, SmoothQuant, SLEB} empirically confirm that such combinations with quantization are both feasible and beneficial.
As independently designed techniques are applied one after another, the order of compression plays a key role.
However, the impact of compression order has not been adequately examined in the current literature.
We denote applying $\compone{\cdot}$ before $\comptwo{\cdot}$ as $\componeonly \rightarrow \comptwoonly$ or $(\comptwoonly \circ \componeonly)(\cdot)$. 

\section{Joint Compression Order Optimization}
\label{sec:problem}

\subsection{Problem Definition}
\label{subsec:problem_definition}
We are given a pre-trained model and multiple compression techniques, each associated with a specific compression rate.
The goal is to find the optimal order in which to sequentially apply these methods.
An order is considered optimal if it minimizes the degradation in model performance.
We quantify performance using a metric $\metric{\cdot}$, where higher values indicate better outcomes (e.g., classification accuracy or the negative of perplexity).
We provide the formal definition as Problem~\ref{problem}.
%
%
\begin{problem}[Joint Compression Order Optimization]
\label{problem}
We have a pre-trained model $\model$, a set of compression methods $\compset = \{\compone{\cdot}, \comptwo{\cdot}, \cdots, \compn{\cdot}\}$, and a performance metric $\metric{\cdot}$.
For a set $\permset = \{ \perm: \compset \rightarrow \compset ~|~ \perm \text{ is bijective}\}$ of all permutations over $\compset$, the goal is to find the optimal permutation $\permopt \in \permset$ that maximizes the performance of the compressed model: $\permopt = \argmax_{\perm \in \permset}  \metric{\perm(\model)}$.
\end{problem}
%
%
\subsection{Characterizing Compression Attributes}
\label{subsec:terms}
Two key attributes arise when characterizing compression in a general setting: granularity and intensity.
Granularity refers to the smallest structural unit
on which compression is applied, and intensity refers to how aggressively the method alters the model, measured by its impact on performance.

\smallsection{Granularity of Compression}
Compression methods are not applied to the model as a whole, but rather operate locally on its individual components.
We define compression granularity as the atomic level at which compression is performed.
To formalize this notion, we begin by abstracting the model into a set of component types, such as layers, sublayers, or attention heads.
We refer to these as \textit{abstract types}, which define the structural units over which compression may act.
For two abstract types $\type_{1}$ and $\type_{2}$, we say $\type_{1}$ is \textit{larger} than $\type_{2}$ if $\type_{1}$ strictly contains $\type_{2}$ as a structural unit.
Among all types that are larger than both $\type_{1}$ and $\type_{2}$, we define the \textit{least upper type} $\lut(\type_{1}, \type_{2})$ as the smallest one.
For a given model $\model$, let $\typeset{\model}$ denote the set of abstract types; this set depends on the model architecture.

Each compression method $\comp{\cdot}$ may be applicable only to a subset of abstract types.
We denote this subset by $\typeset{\componly} \subseteq \typeset{\model}$, representing the structural levels at which $\comp{\cdot}$ can operate.
For instance, layer-wise pruning in large language models is applicable only to units coarser than layers.
Then, the granularity of $\comp{\cdot}$ is the smallest unit $\typegran{\componly} \in \typeset{\componly}$ on which $\comp{\cdot}$ is applicable, as defined in Definition~\ref{def:granularity}.
%
%
\begin{definition}[Compression Granularity]
\label{def:granularity}
	For a model $\model$ with a set $\typeset{\model}$ of abstract types and compression method $\comp{\cdot}$,
	the compression granularity $\typegran{\componly} \coloneqq \argmin_{\type \in \typeset{\componly}}|\type|$, where
	$\typeset{\componly} \subseteq \typeset{\model}$ denotes the set of abstract types on which $\comp{\cdot}$ operates,
	and $|\type|$ denotes the structural size of type $\type$.
\end{definition}
%
%
\smallsection{Intensity of Compression}
Compression methods affect the model differently even at identical compression ratios, so comparing their intensities directly is challenging.
To assess compression strength, we introduce three concepts grounded in performance degradation: performance gap $\pg{\componeonly, \comptwoonly}$, compression equivalent ratio $\compeqrate_{\componly}$, and compression order advantage $\coa{\componeonly \rightarrow \comptwoonly}$.

Performance differences between two methods $\compone{\cdot; \comprate_1}$ and $\comptwo{\cdot; \comprate_2}$, each applied at its respective compression ratios $\comprate_1$ and $\comprate_2$, provide a direct measure of their relative intensity.
We call this the \textit{performance gap} $\pg{\model, \metriconly; \compone{\cdot; \comprate_1}, \comptwo{\cdot; \comprate_2}}$, or simply $\pg{\componeonly, \comptwoonly}$, as defined in Definition~\ref{def:pg}.
If $\pg{\componeonly, \comptwoonly} > 0$, we refer  to $\comptwo{\cdot; \comprate_2}$ as the stronger compression and $\compone{\cdot; \comprate_1}$ as the weaker one.
%
%
\begin{definition}[Performance Gap]
	\label{def:pg}
	Given a model $\model$, a performance metric $\metric{\cdot}$, and two compression methods $\compone{\cdot; \comprate_1}$ and $\comptwo{\cdot; \comprate_2}$,
	the performance gap between two methods
	$\pg{\model, \metriconly; \compone{\cdot; \comprate_1}, \comptwo{\cdot; \comprate_2}}
	\coloneqq \metric{\compone{\phi; \comprate_1}} - \metric{\comptwo{\phi; \comprate_2}}$.
\end{definition}
%
%
Although $\pg{\cdot}$ offers a clear pairwise comparison, its values in metric units are difficult to interpret and may grow rapidly as the compression ratio increases.
Alternatively, mapping methods onto a common scale allows for direct comparison at the level of compression ratios.
While multiple choices exist for the baseline method, we select quantization as it exhibits the best performance across diverse models, thereby offering the widest range.
Accordingly we define the Compression Equivalent Ratio (CER) $\compeqrate(\compone{\cdot}, \quantonly, \comprate)$, or simply $\compeqrate_{\componeonly}$, which expresses the effect of method $\compone{\cdot; \comprate}$ at ratio $\comprate$ as an equivalent ratio of quantization $\quant{\cdot}$, as Definition~\ref{def:cer}.
In other words, starting from a 16-bit model, a compression method $\comp{\model; \comprate}$ with $\compeqrate_{\componly} = 2$ achieves the same performance as 8-bit quantization.
Note that CER of quantization $\quantonly(\cdot)$ is naturally equal to its own compression ratio (i.e., $\compeqrate_{\quantonly} = \compratequant$).
We adopt a straightforward approach by computing CER through linear interpolation.
For instance, $\comp{\cdot}$ achieving $\metric{\componly; \comprate}=$ 65\% accuracy maps to $\compeqrate_{\componly} = 3$ when quantization $\quant{\cdot}$ yields $\metric{\quantonly; \compratequant = 2}=$ 70\% and $\metric{\quantonly; \compratequant = 4}=$ 60\% accuracy, respectively.
%
\begin{definition}[Compression Equivalent Ratio]
	\label{def:cer}
	Given a model $\model$, a performance metric $\metric{\cdot}$, a compression method $\comp{\cdot}$, a quantization method $\quant{\cdot}$, and a compression ratio $\comprate$,
	the compression equivalent ratio $\compeqrate(\comp{\cdot}, \quantonly, \comprate) \coloneqq C'$ such that
	$\metric{\quant{\model; C'}} = \metric{\comp{\model; \comprate}}.$
	%
\end{definition}
%
%

Until now our discussion is limited to single methods; but when multiple methods are applied, how should intensity be defined?
Our scope centers on measuring how intensity changes by compression order.
Accordingly, we capture the gain from applying  $\compone{\cdot}$ before $\comptwo{\cdot}$ over the reverse as compression order advantage $\coa{\model, \metriconly; \compone{\cdot} \rightarrow \comptwo{\cdot}}$, or simply $\coa{\componeonly \rightarrow \comptwoonly}$, as Definition~\ref{def:coa}.
%
%
%
\begin{definition}[Compression Order Advantage]
	\label{def:coa}
	Given a model $\model$, a performance metric $\metric{\cdot}$, and two compression methods $\compone{\cdot; \comprate_1}$ and $\comptwo{\cdot; \comprate_2}$,
	the compression order advantage
	$\coa{\model, \metriconly; \compone{\cdot; \comprate_1} \rightarrow \comptwo{\cdot; \comprate_2}}
	\coloneqq \pg{\componeonly \rightarrow \comptwoonly, \comptwoonly \rightarrow \componeonly}
	= \metric{(\comptwoonly \circ \componeonly)(\phi)} - \metric{(\componeonly \circ \comptwoonly)(\phi)}.$
\end{definition}
%

\subsection{The Progressive Intensity Hypothesis}
\label{subsec:hypothesis}
Our goal is to uncover general patterns in how compression order affects the model performance in joint compression scenarios.
While prior works have focused primarily on isolated settings, we seek to establish a broadly applicable principle.
To this end, we propose \textit{the Progressive Intensity Hypothesis}, which posits that applying stronger compression methods at later stages generally yields better performance.
We formalize this hypothesis for a pair of methods in Hypothesis~\ref{hypothesis}, which serves as the main focus of our analysis;
its extension to multiple methods is presented in Appendix~\ref{app:subsec:theory_generalize}.
%
%
\begin{hypothesis}[The Progressive Intensity Hypothesis]
\label{hypothesis}
	Let $\compone{\cdot; \comprate_1}$ and $\comptwo{\cdot; \comprate_2}$ be two compression methods applied to a model $\model$.
	Then, the compression order advantage
	$\coa{\componeonly \rightarrow \comptwoonly}$
	grows monotonically with the performance gap $\pg{\componeonly, \comptwoonly}$,
	or equivalently with the CER difference
	$\compeqrate_{\comptwoonly} - \compeqrate_{\componeonly}.$
%
\end{hypothesis}
\vspace{-2mm}
%
As an example, if methods $\compone{\cdot}$ and $\comptwo{\cdot}$ yields $\metric{\componeonly; \comprate_1}=$ 75\% and $\metric{\comptwoonly; \comprate_2}=$ 70\% accuracy, respectively (i.e., $\pg{\componeonly, \comptwoonly} = 5\%$p), the compression order advantage $\coa{\componeonly \rightarrow \comptwoonly}$ is mild;
replacing $\comprate_2$ into $\comprate'_2$ at $\metric{\comptwoonly; \comprate'_2}=$ 60\% accuracy (i.e., $\pg{\componeonly, \comptwoonly} = 15\%$p) results in a larger advantage.
\vspace{-2mm}

\section{Theoretical Analysis}
\label{sec:theory}
We theoretically analyze how compression order affects the model performance.
We introduce disjoint selectivity to isolate order-dependent units, and prove in Theorem~\ref{thm:1} that only these units determine the performance gap.
We then show in Theorem~\ref{thm:2} that Hypothesis~\ref{hypothesis} holds due to the reduction of order-dependent units.
We later extend to non-disjoint cases in which interference occurs.
Consistent with earlier works~\citep{Interplay}, we investigate each unit, relying on Assumption~\ref{assumption:general}.

\begin{assumption}
	\label{assumption:general}
	Given a model $\model$ with a set $\layers$ of layers, performance metric $\metric{\cdot}$, a compression method $\comp{\cdot}$, and the layer-wise reconstruction loss $\erroru{\componly}{\layeri}$,
	assume the followings:
	%
	\begin{itemize}[leftmargin=3.4mm, itemsep=-1mm, topsep=0mm]
		\item \textbf{Layer-wise independence.}
	  	The reconstruction error at one layer does not affect the reconstruction error at another:
		$\forall \layeri, \layerj \in \layers, ~ i \neq j: ~ {\partial \, \erroru{\componly}{\layeri}} / {\partial \, \erroru{\componly}{\layerj}} = 0.$
		\item \textbf{Error-performance trade-off.}
  		Model performance is inversely related to total reconstruction error:
		$\exists \beta > 0, ~
		\metric{\model} - \metric{\comp{\model}} = \beta \cdot \sum_{\layeri \in \layers} {\| \erroru{\componly}{\layeri} \| _F^2}. $
	\end{itemize}
\end{assumption}
%

%
\smallsection{Disjoint Selectivity}
Sequential application of two compression methods leads to two distinct scenarios: either there exist units altered by both methods, or all units are exclusively assigned to one.
We define the latter scenario as the case where \textit{disjoint selectivity} holds, as in Definition~\ref{def:disjoint_selectivity}.
This means that while the assignment may vary with order, each unit is ultimately handled by only one method.

%
\begin{definition}[Disjoint Selectivity]
	\label{def:disjoint_selectivity}
	Given a model $\model$, two compression methods $\compone{\cdot}$ and $\comptwo{\cdot}$ with respective granularities $\typegran{\componeonly}$ and $\typegran{\comptwoonly}$,
	disjoint selectivity holds if $~\forall \uniti \in \units{\model; \lut(\typegran{\componeonly}, \typegran{\comptwoonly})},
	~ \forall \pi \in \{\componeonly \circ \comptwoonly, \comptwoonly \circ \componeonly\},
	~ \applies{\uniti}{\componeonly}{\pi} + \applies{\uniti}{\comptwoonly}{\pi} = 1,$
	where $\units{\model; \type}$ is the set of all units of model $\model$ at granularity $\type$,
	and $\applies{\unit}{\componly}{\pi}$ denotes whether $\comp{\cdot}$ modifies unit $\unit$ under the order $\pi$ (i.e., 1 if modified, 0 otherwise).
	\vspace{-1mm}
\end{definition}
%

Under disjoint selectivity, the compression order advantage $\coa{\componeonly \rightarrow \comptwoonly}$ is proportional to the cumulative sum of error difference $g(\cdot)$ across units assigned differently depending on the order as formulated in Theorem~\ref{thm:1}.
The underlying intuition is that the performance gap rises solely from units whose assignment varies with the order; for others, the error remains invariant and thus cancels out.
To illustrate, consider units $\unit_1$, $\unit_2$, and $\unit_3$ and compression methods $\compone{\cdot}$ and $\comptwo{\cdot}$.
If $\unit_1$ is always handled by $\compone{\cdot}$ regardless of the order, while $\unit_2$ and $\unit_3$ are assigned differently depending on the order, then the advantage $\coa{\componeonly \rightarrow \comptwoonly}$ is proportional to error difference of units $\unit_2$ and $\unit_3$.


\begin{theorem}[Compression Order Advantage under Disjoint Selectivity]
\label{thm:1}
	Suppose we compress a model $\model$ with two compression methods $\compone{\cdot}$ and $\comptwo{\cdot}$ with respective granularities $\typegran{\componeonly}$ and $\typegran{\comptwoonly}$, where disjoint selectivity holds.
	Then, under Assumption~\ref{assumption:general},
	the compression order advantage
	
	$\coa{\componeonly \rightarrow \comptwoonly}
	= \metric{(\comptwoonly \circ \componeonly)(\model)} - \metric{(\componeonly \circ \comptwoonly)(\model)}$ equals to $\beta \cdot \big(\sum_{\uniti \in \grouptwo} {g(\uniti)} - \sum_{\uniti \in \groupone} {g(\uniti)} \big)$,
	
	%
	%
	where $\beta$ is the coefficient between the model performance and total error induced from Assumption~\ref{assumption:general},
	$g(\uniti) = \big\| \erroru{\componeonly}{\uniti} \big\|_F^2 - \big\| \erroru{\comptwoonly}{\uniti} \big\|_F^2$ is the error gap according to the method applied,
	and $\groupone = \{\unit ~|~ \applies{\unit}{\componeonly}{\comptwoonly \circ \componeonly} = 1, ~ \applies{\unit}{\componeonly}{\componeonly \circ \comptwoonly} = 0 \}$
	and $\grouptwo = \{\unit ~|~ \applies{\unit}{\componeonly}{\comptwoonly \circ \componeonly} = 0, ~ \applies{\unit}{\componeonly}{\componeonly \circ \comptwoonly} = 1\}$ are groups of order-dependent units.
\end{theorem}
\begin{proof}
    Refer to Appendix~\ref{app:subsec:proof_thm1}.
\end{proof}


\smallsection{Monotonicity}
Under disjoint selectivity, we show that Hypothesis~\ref{hypothesis} holds when the two compression methods are \textit{well-designed}—that is, minimally disruptive to the model.
We examine this through a case study on pruning and quantization.
We assume a favorable scenario where pruning is configured to induce minimal degradation, and quantization introduces symmetric, zero-mean errors centered at the original values.
These assumptions are formalized in Assumption~\ref{assumption:pandq}.


\begin{assumption}
	\label{assumption:pandq}
	Given a model $\model$ with a set $\layers$ of layers and performance metric $\metric{\cdot}$, assume that:
	%
	\begin{itemize}[leftmargin=3.4mm, itemsep=-1mm, topsep=0mm]
		\item \textbf{Well-designed pruning $\prune{\cdot}$.}
		The pruning method is chosen from the set of pruning strategies that aim to preserve the model performance:
			$\prune{\cdot} \in \allprunes{\comprateprune}$
		where $\allprunes{\comprateprune} = \{\prunei{\cdot} | \comprate(\prunei{\model}) = \comprateprune,\ \metric{\model} - \metric{\prunei{\model}} \le \delta\}$ denotes the set of pruning strategies that satisfy the target ratio $\comprateprune$ while keeping performance degradation within a small budget $\delta$.
		\item \textbf{Well-designed quantization $\quant{\cdot}$.}
		For all layers, quantized outputs follow a symmetric distribution around the original values: 		
		$\forall \layeri \in \layers, \quant{\weighti}\quant{\acti} \sim \gaussian{\weighti \acti, \sigma_{\quantonly}^2\mat{I}},$ where $\gaussian{\cdot}$ is the Gaussian distribution.
		The quantization error is negligible (i.e., $\quant{\weighti}\quant{\acti} - \weighti\acti \ll \weighti\acti$).
	\end{itemize}
	\vspace{-2mm}
\end{assumption}
%

Theorem~\ref{thm:2} states that when disjoint selectivity holds and the compression methods are well-designed, $\coa{\quantonly \rightarrow \pruneonly}$ increases monotonically with CER difference\footnote{
		By definition, quantization serves as the baseline scale, so its CER equals its own compression ratio (i.e., $\compeqrate_{\quantonly} = \compratequant$).
		Therefore, $\compeqrateprune - \compratequant$ represents the CER difference between pruning and quantization.}
$\compeqrateprune - \compratequant$ for fixed $\comprateprune$.
Note that as $\compratequant$ decreases (i.e., $\metric{\quant{\model}}$ increases), both CER difference $\compeqrateprune - \compratequant$ and performance gap $\pg{\quantonly,\pruneonly}$ increase monotonically.
We show that $\coa{\quantonly \rightarrow \pruneonly}$ increases in this setting because the gap depends solely on order-dependent units under Theorem~\ref{thm:1}. 
We discuss the impact of $\comprateprune$ in Appendix~\ref{app:subsec:violation}.


\begin{theorem}[Monotonicity]
	\label{thm:2}
	Suppose we compress a model $\model$ with pruning $\prune{\cdot}$ and quantization $\quant{\cdot}$, where disjoint selectivity holds.
	Then, under Assumptions~\ref{assumption:general} and~\ref{assumption:pandq},
	given performance metric $\metric{\cdot}$ and two pairs of compression ratios $(\compratepruneone, \compratequantone)$ and $(\compratepruneone, \compratequanttwo)$,
	if CER difference increases
	\vspace{-1mm}
	\begin{equation*}
		\compeqratepruneone - \compratequantone > \compeqratepruneone - \compratequanttwo,
		\vspace{-1mm}
	\end{equation*}
	then, the compression order advantage increases monotonically:
	\vspace{-1mm}
	\begin{equation*}
		\coa{\model, \metriconly; \quant{\cdot; \compratequantone} \rightarrow \prune{\cdot; \compratepruneone}} \geq
		\coa{\model, \metriconly; \quant{\cdot; \compratequanttwo} \rightarrow \prune{\cdot; \compratepruneone}}.
		\vspace{-3mm}
	\end{equation*}
\end{theorem}
\begin{proof}
	Refer to Appendix~\ref{app:subsec:proof_thm2}.
\end{proof}
\vspace{-2mm}


\smallsection{Granularity and Interference}
Disjoint selectivity does not always hold in practical joint compression settings for pruning and quantization.
As pruning operates by fully discarding or keeping each unit, it always satisfies disjoint selectivity.
In contrast, quantization satisfies this condition only when its granularity is finer than or equal to that of pruning.
Figure~\ref{fig:interference} illustrates this:
(a) if $\typegran{\pruneonly} \geq \typegran{\quantonly}$, disjoint selectivity is preserved as pruning removes entire quantization units.
However, (b) if $\typegran{\pruneonly} < \typegran{\quantonly}$, pruning may partially eliminate a quantization unit, introducing regions where both methods interfere.

\begin{figure}[t]
	\centering
	\vspace{-2mm}
	\includegraphics[width=1\linewidth]{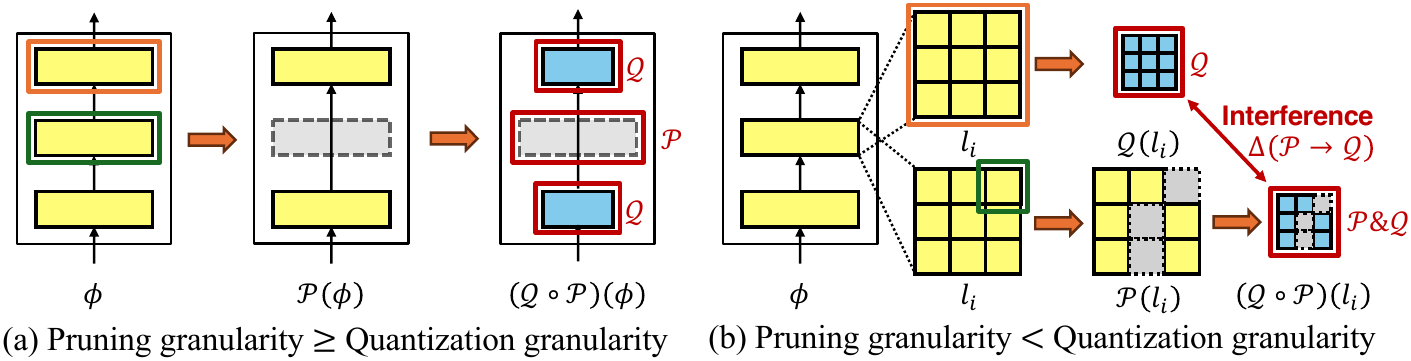}
	\vspace{-5mm}
	\caption{
		A case study of pruning $\prune{\cdot}$ and quantization $\quant{\cdot}$ on model $\model$.
		(a) if pruning granularity (green) is coarser or equal to quantization granularity (orange), disjoint selectivity holds.
		(b) Otherwise, partial removal of quantization units by pruning introduces extra error, termed \textit{interference} $\interferenceonly$.
	}
	\label{fig:interference}
	\vspace{-2mm}
\end{figure}

In general joint compression of two methods $\compone{\cdot}$ and $\comptwo{\cdot}$, this violation of disjoint selectivity
introduces additional error, which we define as  \textit{interference} $\interference{\model; \componeonly \rightarrow \comptwoonly}$, or simply $\interference{\componeonly \rightarrow \comptwoonly}$ in Definition~\ref{def:interference}.
Intuitively, interference quantifies how one method disturbs the behavior of the other.
%
%
\begin{definition}[Interference]
	\label{def:interference}	
	Given a model $\model$ and two methods $\compone{\cdot}$ and $\comptwo{\cdot}$,
	the interference
	%
	\vspace{-1mm}
	\begin{equation*}
		\interference{\model; \componeonly \rightarrow \comptwoonly}
		\coloneqq \sum_{\unit \in \mathbb{X}}
		\big( \erroru{\comptwoonly \circ \componeonly}{\unit} - \erroru{\comptwoonly}{\unit} \big),
		\text{where}~
		\mathbb{X} = \units{\model; \typegran{\comptwoonly}} ~\cap~ \{ \unit ~|~ \applies{\unit}{\comptwoonly}{\comptwoonly \circ \componeonly} = 1 \},
		\vspace{-1mm}
	\end{equation*}
	set $\units{\model; \type}$ contains all units of model $\model$ at type $\type$,
	$\applies{\unit}{\componly}{\pi}$ indicates whether unit $\unit$ is modified by $\comp{\cdot}$ under order $\pi$
	(1 if modified, 0 otherwise),
	and $\erroru{\comp{\cdot}}{\unit}$ denotes the error on $\unit$ after applying $\comp{\cdot}$.
	
\end{definition}

Interference may or may not occur, depending on its applied techniques.
A notable example is mixed-precision quantization, where treating each bit-width quantization as a distinct method satisfies disjoint selectivity, thereby avoiding interference.
In our primary focus of pruning and quantization, the magnitude of interference is determined by the pruning ratio $p$.
When interference occurs, quantization operates on weights altered by pruning, thus the deviation from the original distribution scales with the pruning ratio.
As pruning ratio $p$ increases, a larger portion of weights is removed before quantization, leading to stronger interference.
This additional error depends only on pruning and enters additively into the order advantage, and thus remains independent of quantization strength while preserving the monotonic trend.
Consequently, while exact outcomes may vary, $\coa{\componeonly \rightarrow \comptwoonly}$ remains a monotonic function of $\compeqrateprune - \compratequant$ even under interference.
In conclusion, Hypothesis~\ref{hypothesis} holds under both disjoint and interfering scenarios, highlighting its general validity. 

\section{Experimental Findings}
\label{sec:experiment}
\begin{figure}[t]
	\vspace{-3mm}
	\centering
	\includegraphics[width=0.9\linewidth]{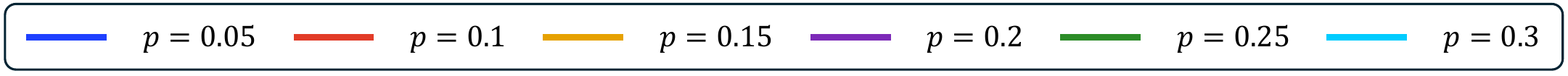}\\
	\hspace{0.01\linewidth}
	\begin{subfigure}[t]{0.32\linewidth}
		\centering
		\includegraphics[width=\linewidth]{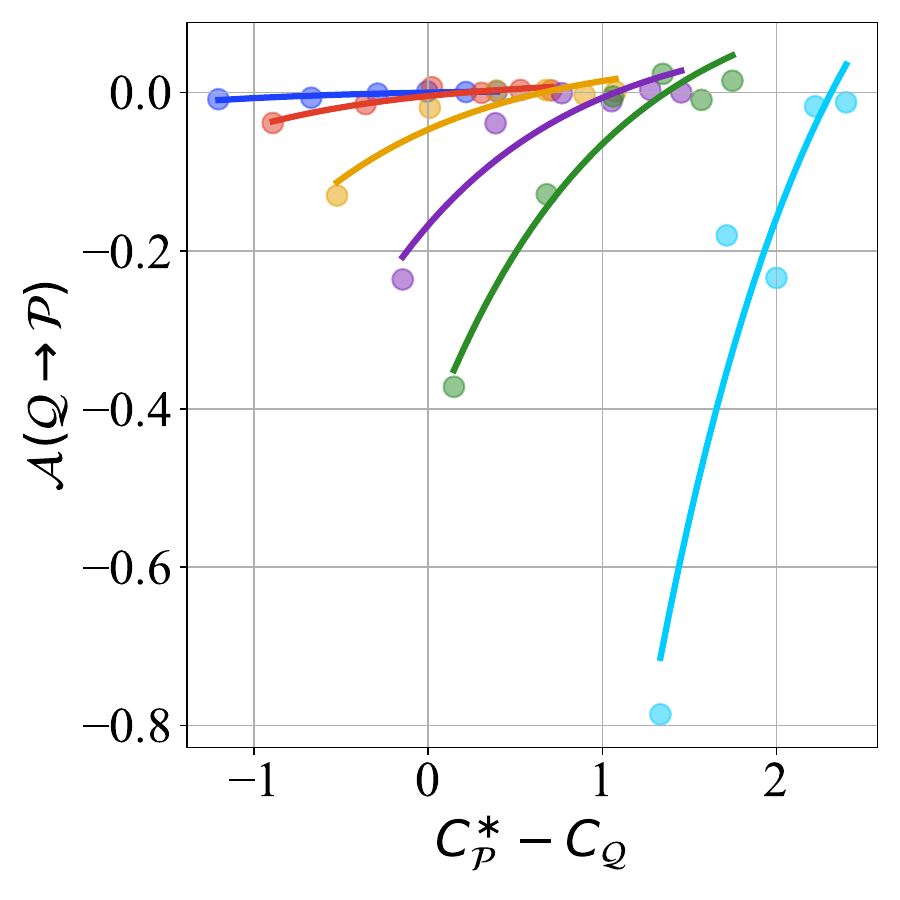}
		\caption{LLaMA 2 7B}
		\label{fig:llm_result:llama2_7B}
	\end{subfigure}
	\hfill
	\begin{subfigure}[t]{0.31\linewidth}
		\centering
		\includegraphics[width=\linewidth]{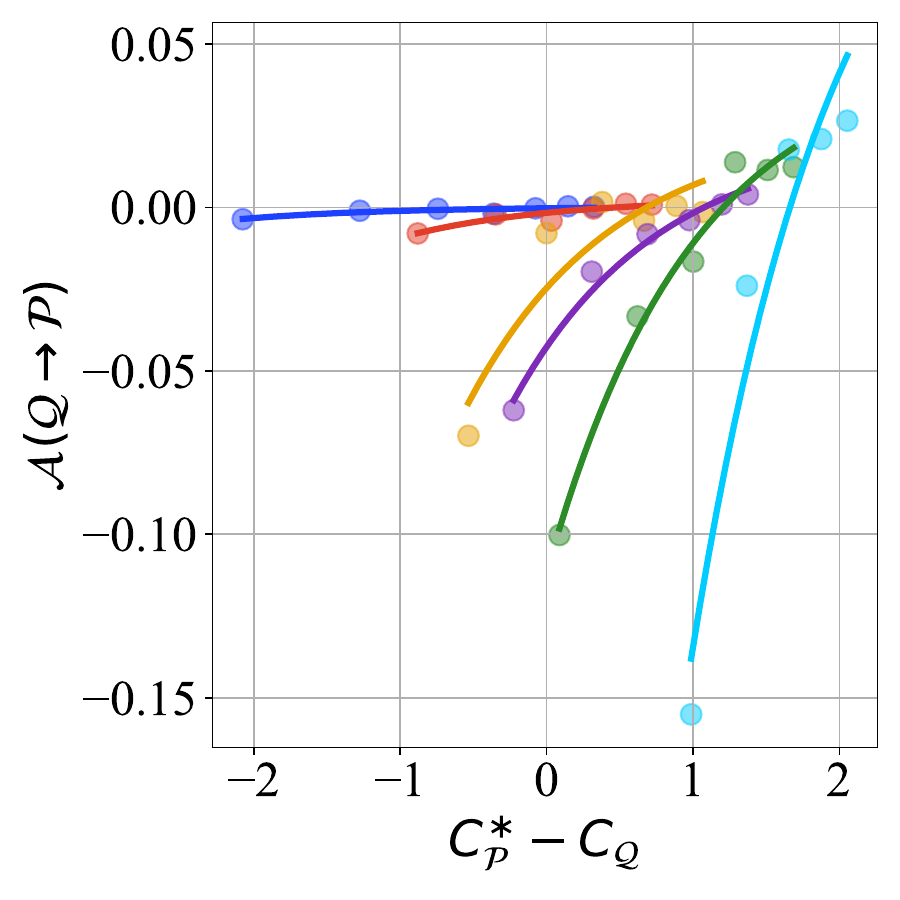}
		\caption{LLaMA 2 13B}
		\label{fig:llm_result:llama2_13B}
	\end{subfigure}%
	\hfill
	\begin{subfigure}[t]{0.31\linewidth}
		\centering
		\includegraphics[width=\linewidth]{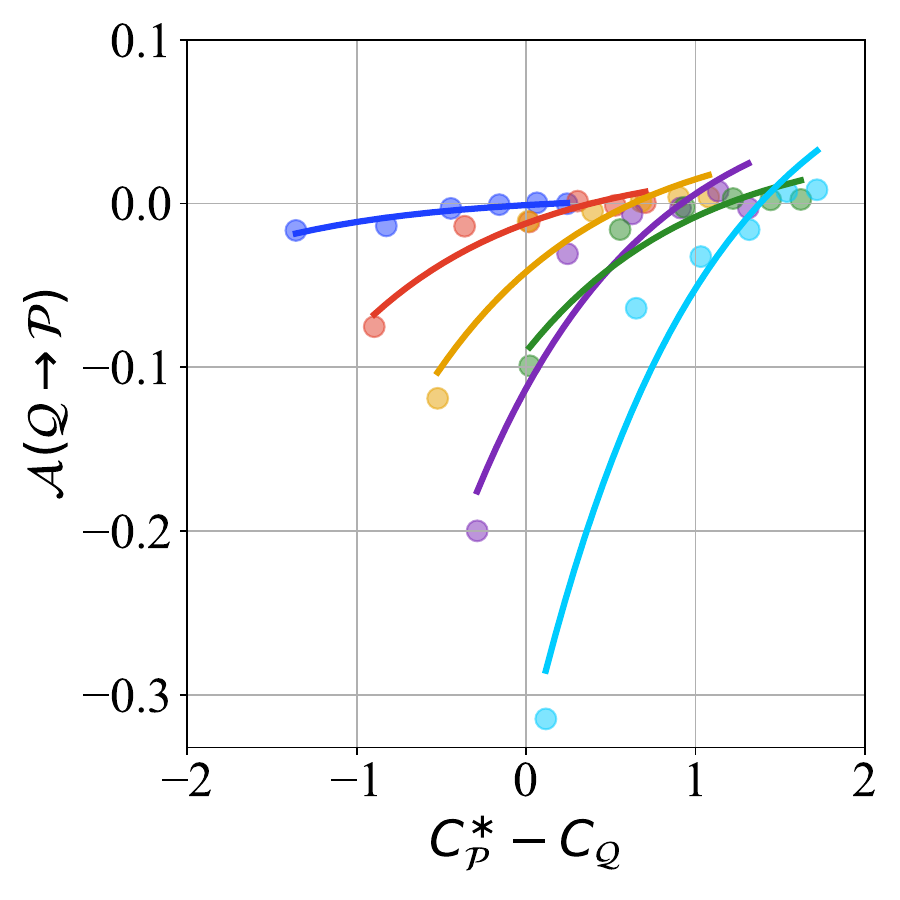}
		\caption{LLaMA 3 8B}
		\label{fig:llm_result:llama3_8B}
	\end{subfigure}
	\vspace{-1mm}
	\caption{
		Across diverse language models, the compression order advantage $\coa{\quantonly \rightarrow \pruneonly}$ increases monotonically with the CER difference $\compeqrateprune - \compratequant$.
		See Section~\ref{subsec:experiment_llms} for details.
	}
	\label{fig:llm_result}
	\vspace{-3mm}
\end{figure}

We empirically validate our hypothesis in joint compression scenarios, starting with pruning and quantization on language and vision models.
We then extend to general pipelines beyond them.

\vspace{-1mm}
\subsection{Experimental Setup}
\label{subsec:setup}
We briefly introduce the experimental setup.
Further setups are detailed in Appendix~\ref{app:sec:setup}.

\smallsection{Setup}
For language models, we focus on decoder-only LLMs, mainly LLaMA~\citep{LLaMA2} models.
The main metric is the negative of perplexity on WikiText-2~\citep{WikiText2} dataset; results on
commonsense reasoning tasks appear in
Appendix~\ref{app:subsec:csr}.
For vision models, we evaluate the classification accuracy of ResNet-18~\citep{ResNet} (CNNs) and DeiT-Base~\citep{DeiT} (ViTs) models on ImageNet~\citep{ImageNet} dataset.

\smallsection{Baselines}
We evaluate three pruning (SparseGPT~\citep{SparseGPT}, Wanda~\citep{Wanda}, and SLEB~\citep{SLEB}) and four quantization methods (RTN~\citep{RTN}, OPTQ~\citep{OPTQ}, QuaRot~\citep{QuaRot}, and QuaRot + OPTQ) for language models.
For vision models, we apply PRACTISE~\citep{PRACTISE} and N2UQ~\citep{N2UQ} for CNNs, and adopt SAViT~\citep{SAViT} and RepQ-ViT~\citep{RepQViT} for ViTs.

\vspace{-1mm}
\subsection{Analysis on Language Models}
\label{subsec:experiment_llms}

We verify whether Hypothesis~\ref{hypothesis} holds for language models.
Then, we analyze the effect of weight updates and rotations in quantization, and investigate the impact of pruning granularity on interference.

\smallsection{Compression Order Advantage by CER Differences}
We analyze how the compression order advantage $\coa{\quantonly \rightarrow \pruneonly}$ varies with CERs for SparseGPT ($\prune{\cdot}$) and QuaRot ($\quant{\cdot}$) across three models: LLaMA 2 7B, 13B, and LLaMA 3 8B.
Figure~\ref{fig:llm_result} confirms that $\coa{\quantonly \rightarrow \pruneonly}$ increases monotonically with the CER difference for all three models.
Each point in the figure corresponds to a compression ratio pair ($\comprateprune, \compratequant$), defined by pruning ratio $p \in [0.05, 0.1, 0.15, 0.2, 0.25, 0.3]$ and quantization bit-width $B_{\quantonly} \in [4, 5, 6, 7, 8]$.
We fit an exponential curve per pruning ratio $p$, reflecting the underlying trend.
In this setting, the x-axis $\compeqrateprune - \compratequant$ captures the difference between the intrinsic intensities of pruning and quantization, with both $\compeqrateprune$ and $\compratequant$ calibrated on the original model $\model$.
Consistent trends across diverse language model architectures and scales supports the validity of Hypothesis~\ref{hypothesis}; see Appendices~\ref{app:subsec:encoder_based} and~\ref{app:subsec:llm_results} for results on encoder-based and other decoder-only models, respectively.


\begin{tcolorbox}[colframe=black!60, colback=lightgrey, boxrule=1pt, arc=3pt, left=2pt, right=2pt, top=2pt, bottom=2pt]
\textbf{Finding 1.} The Progressive Intensity Hypothesis holds across diverse language models of varying scales and architectures, showing that stronger compression should be applied later.
\end{tcolorbox}


\smallsection{Weight Updates and Rotation-based Transformations}
We investigate the hypothesis under practical techniques such as weight-updates and rotations.
Figure~\ref{fig:llm_result_update_and_rotation} shows that Hypothesis~\ref{hypothesis} consistently holds across diverse combinations of methods.
Our framework is agnostic to the type of compression methods;
weight updates and rotations reduce quantization error, thereby increasing $\compeqrateprune$.

\begin{figure}[t]
	\centering
	\includegraphics[width=0.9\linewidth]{main_result_legend}\\
	\includegraphics[width=0.93\linewidth]{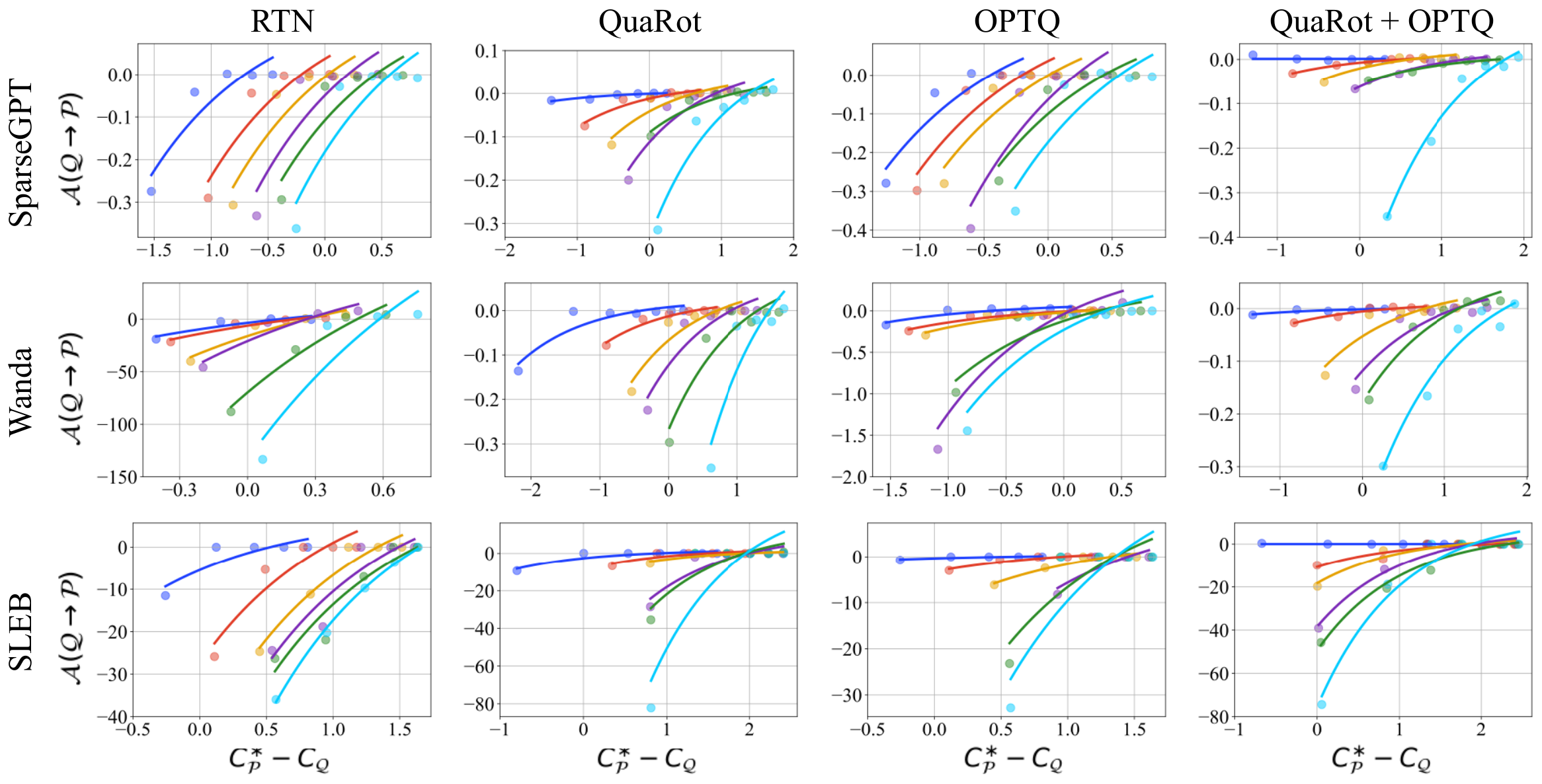}
	\vspace{-2mm}
	\caption{
		Compression order advantage $\coa{\quantonly \rightarrow \pruneonly}$ against CER difference $\compeqrateprune - \compratequant$ for three pruning $\prune{\cdot}$ and four quantization $\quant{\cdot}$ methods on a LLaMA 3 8B model.
		Our hypothesis consistently holds for language models regardless of pruning granularity, rotation, and weight updates.
	}
	\label{fig:llm_result_update_and_rotation}
	\vspace{-3mm}
\end{figure}

%
\begin{tcolorbox}[colframe=black!60, colback=lightgrey, boxrule=1pt, arc=3pt, left=2pt, right=2pt, top=2pt, bottom=2pt]
\textbf{Finding 2.} The hypothesis generalizes beyond specific design choices of pruning and quantization, remaining robust under weight-update and rotation methods.
\end{tcolorbox}
%

%
\begin{wrapfigure}[11]{R}{3.1cm}
	\vspace{-4mm}
	\centering
	\includegraphics[width=\linewidth]{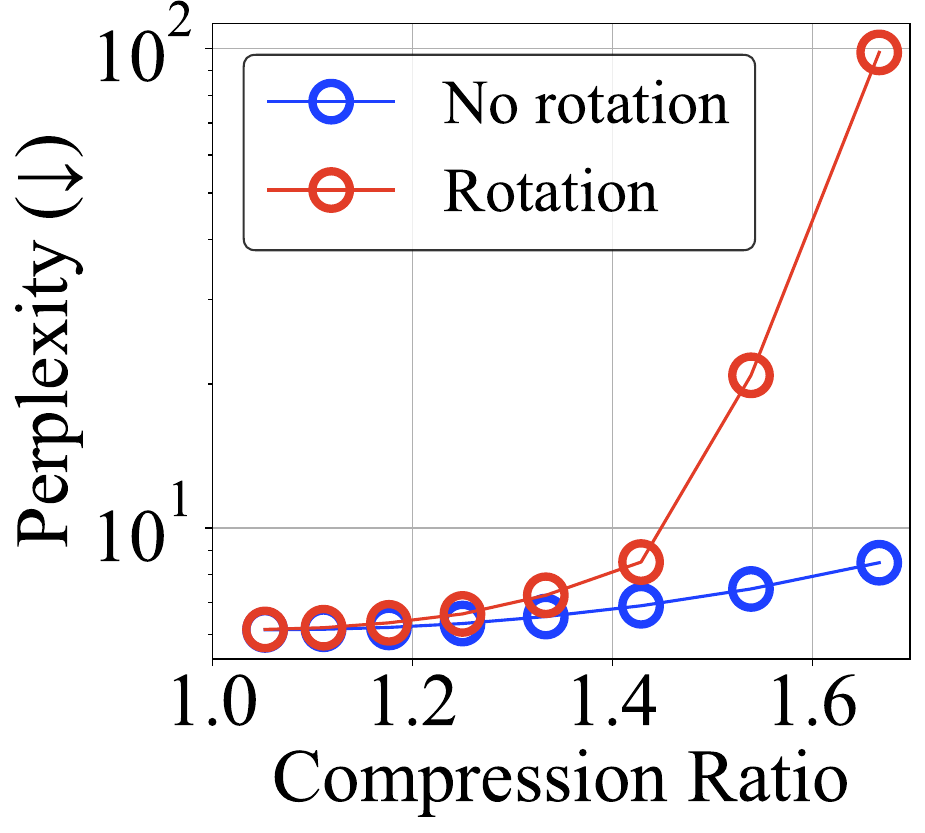}
	\vspace{-6mm}
	\caption{
		Rotation impact on pruning.
		See Section~\ref{subsec:experiment_llms} for details.
	}
	\label{fig:rotation_pruning}
\end{wrapfigure}
%

An intriguing phenomenon arises in pruning rotation-based methods: applying pruning after rotation results in a drastic performance drop compared to pruning without rotation.
Figure~\ref{fig:rotation_pruning} illustrates the perplexity changes of LLaMA 3 8B model pruned by SparseGPT, depending on the Hadamard rotation from QuaRot.
Increasing the pruning ratio amplifies the discrepancy between rotated and non-rotated settings, as pruning is applied without accounting for rotation.
This effect emerges because rotation may introduce two types of errors: \textit{matrix-wise errors} from residual components and \textit{element-wise errors} from altered pruning decisions.
We further discuss the details in  Appendix~\ref{app:subsec:rotation_pruning} and Table~\ref{app:tab:rotation_pruning}.
As rotation may intensify pruning, it is essential to design pruning approaches compatible with rotation-based quantization, the emerging \textit{de facto} standard.

%
\begin{tcolorbox}[colframe=black!60, colback=lightgrey, boxrule=1pt, arc=3pt, left=2pt, right=2pt, top=2pt, bottom=2pt]
\textbf{Finding 3.}
Rotation amplifies pruning effects, underscoring the necessity of designing rotation-aware pruning methods.
\end{tcolorbox}
%


\begin{wraptable}[10]{r}{4.8cm}
	\vspace{-0.5mm}
	\centering
	\caption{
		$\coa{\quantonly \rightarrow \pruneonly}$ by quantization ratio $\compratequant$.}
	\renewcommand{\arraystretch}{1.1}
	\resizebox{\linewidth}{!}{
		\begin{tabular}{ccc}
			\hline
			\toprule
			$\compratequant$ $(B_{\quantonly})$ & SparseGPT & SLEB \\
			\midrule
			1.78 (9) & 0.002 & 0 \\
			2.00 (8) & 0.001 & 0 \\
			2.28 (7) & -0.003 & 0 \\
			2.68 (6) & -0.013 & 0 \\
			3.20 (5) & -0.017& -0.057 \\
			4.00 (4) & -49.899 & -9.379 \\
			\bottomrule
			\hline
		\end{tabular}
	}
	\label{tab:interference}
\end{wraptable}

\smallsection{Pruning Granularity and Interference}
We verify the presence of interference by comparing results across different pruning granularities.
Table~\ref{tab:interference} reports $\coa{\quantonly \rightarrow \pruneonly}$ across different quantization bit-width $B_{\quantonly}$ when applying two 5\% pruning methods with QuaRot.
For SLEB, which applies structured pruning at sublayer level, there exists a regime where no layers differ in their pruning status across orders, leading to an exact $\coa{\quantonly \rightarrow \pruneonly}$ of zero (i.e., no interference).
By contrast, SparseGPT, as an unstructured pruning method, exhibits interference in low $\compratequant$ ranges.
Notably, empirical results suggest that interference also exhibits a monotonic trend regarding $\compratequant$.

\vspace{1mm}
\begin{tcolorbox}[colframe=black!60, colback=lightgrey, boxrule=1pt, arc=3pt, left=2pt, right=2pt, top=2pt, bottom=2pt]
\textbf{Finding 4.} Pruning granularity determines interference: structured pruning shows no interference in early regimes, while unstructured pruning exhibits monotonic interference.
\end{tcolorbox}


\noindent
\begin{figure}
	\begin{minipage}[t]{0.64\linewidth}
		\centering
		\begin{subfigure}[t]{0.48\linewidth}
			\centering
			\includegraphics[width=\linewidth]{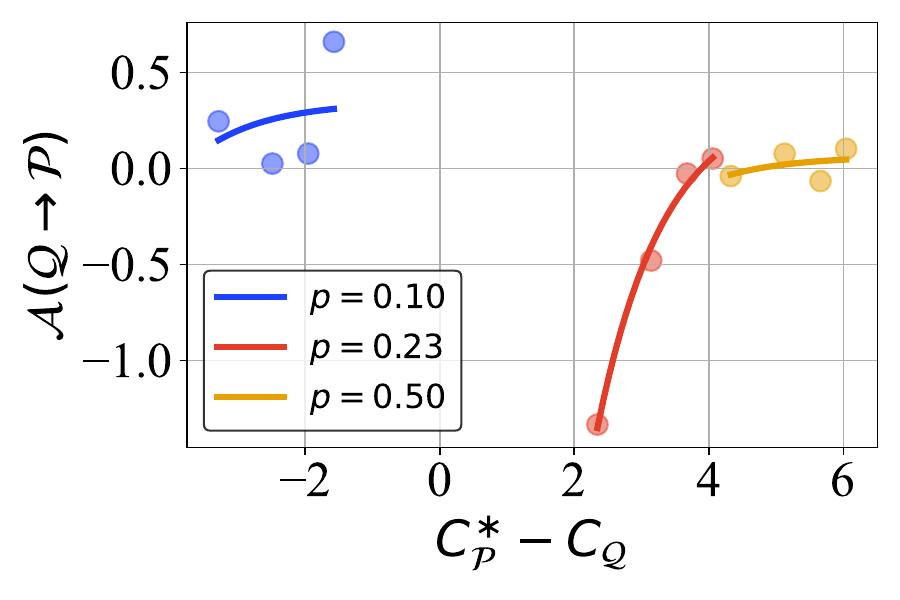}
			\caption{ResNet-18 (CNN)}
			\label{fig:vision:cnn}
		\end{subfigure}
		\hfill
		\begin{subfigure}[t]{0.48\linewidth}
			\centering
			\includegraphics[width=\linewidth]{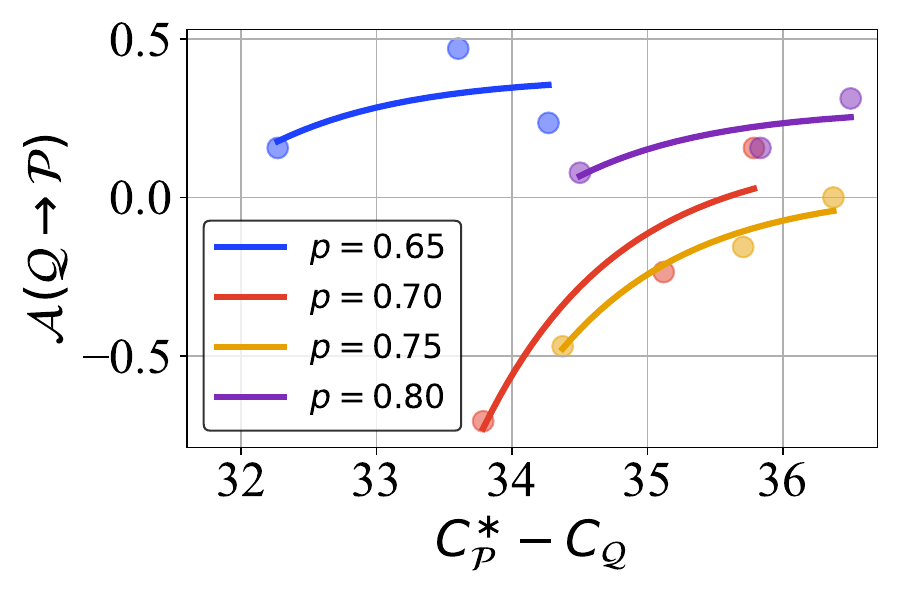}
			\caption{DeiT-Base (ViTs)}
			\label{fig:vision:vit}
		\end{subfigure}
		\vspace{-1mm}
		\caption{The Progressive Intensity Hypothesis holds across diverse vision models.
			See Section~\ref{subsec:experiment_vision} for details.}
		\label{fig:vision}
	\end{minipage}
	\hfill
	\begin{minipage}[t]{0.34\linewidth}
		\vspace{-29mm}
		\centering
		\includegraphics[width=\linewidth]{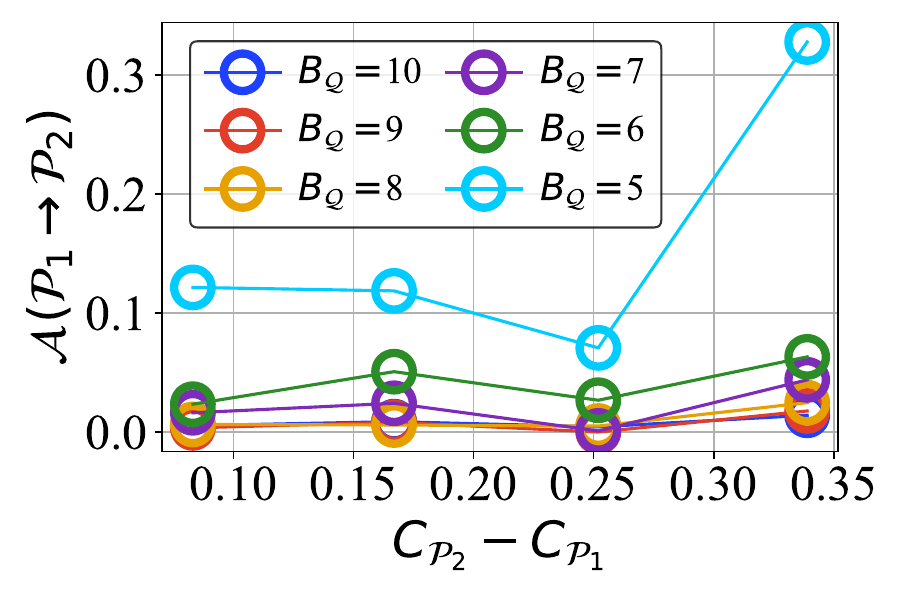}
		\vspace{-2mm}
		\caption{Multi-stage compression results on LLaMA 3 8B model.}
		\label{fig:multi_stage}
	\end{minipage}
\end{figure}
\vspace{-2mm}
%

\vspace{-2mm}
\subsection{Analysis on Vision Models}
\label{subsec:experiment_vision}
\smallsection{CNN and ViT Models}
We verify whether Hypothesis~\ref{hypothesis} holds for vision models, focusing on CNNs and ViTs.
In Figure~\ref{fig:vision}, we analyze the behavior of ResNet-18 and DeiT-Base models under PRACTISE ($\prune{\cdot}$) and N2UQ ($\quant{\cdot}$), and SAViT ($\prune{\cdot}$) and RepQ-ViT ($\quant{\cdot}$) methods, respectively.
The results confirm that both $\coa{\quantonly \rightarrow \pruneonly}$ and $\compeqrateprune - \compratequant$ increase monotonically for both models, regardless of the pruning or quantization configurations.
Notably, the compression order advantage is substantially larger than that observed in language models, where it is often marginally positive.

\vspace{1mm}
\begin{tcolorbox}[colframe=black!60, colback=lightgrey, boxrule=1pt, arc=3pt, left=2pt, right=2pt, top=2pt, bottom=2pt]
\textbf{Finding 5.}
Vision models consistently satisfy the Progressive Intensity Hypothesis regardless of architecture and applied compression techniques.
\end{tcolorbox}
%

\subsection{Beyond Pruning and Quantization: Toward General Pipelines}
\label{subsec:beyond}


We extend the Progressive Intensity Hypothesis to general compression pipelines.
These results align with the $n$-method ordering formulation in Appendix~\ref{app:subsec:theory_generalize}.

\smallsection{Multi-stage Compression}
Pruning is generally performed in multiple stages to mitigate performance degradation. In Figure~\ref{fig:multi_stage}, we investigate the impact of compression order by alternately applying SparseGPT ($\prune{\cdot}$) and QuaRot ($\quant{\cdot}$) to the LLaMA 3 8B model where the sum of pruning ratios $p_1 + p_2 = 0.3$ (e.g., $\prune{\cdot; \compratepruneone} \rightarrow \quant{\cdot} \rightarrow \prune{\cdot; \comprateprunetwo}$).
Our results consistently demonstrate positive advantages, indicating that stronger pruning placed later improves performance under fixed quantization, confirming that our hypothesis holds not only for two stages but also for multiple ones.

\vspace{1mm}
\begin{tcolorbox}[colframe=black!60, colback=lightgrey, boxrule=1pt, arc=3pt, left=2pt, right=2pt, top=2pt, bottom=2pt]
	\textbf{Finding 6.}
	Beyond pairwise orders, the hypothesis holds in practical multi-stage compression, indicating that scheduling stronger compression later in the sequence yields higher accuracy.
\end{tcolorbox}
\vspace{1mm}


\noindent
\begin{figure}
	\centering
\includegraphics[width=0.9\linewidth]{main_result_legend}\\
	\begin{minipage}[t]{0.31\linewidth}
		\centering
		\includegraphics[width=\linewidth]{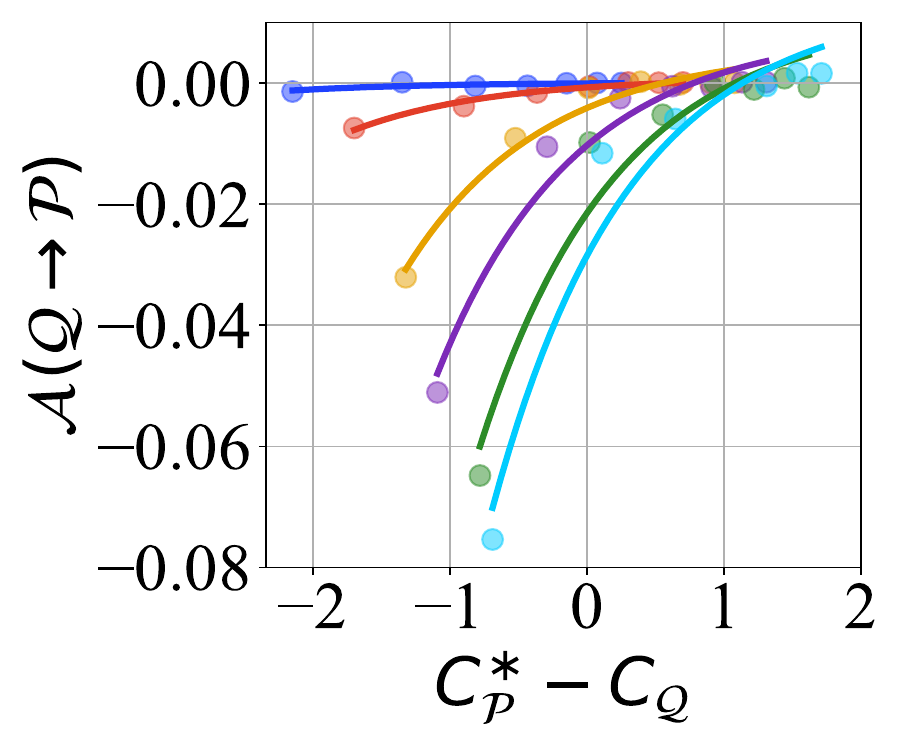}
 		\caption{Joint compression results with LoRA adapters.}
		\label{fig:peft}
	\end{minipage}
	\hfill
	\begin{minipage}[t]{0.31\linewidth}
		\centering
		\includegraphics[width=\linewidth]{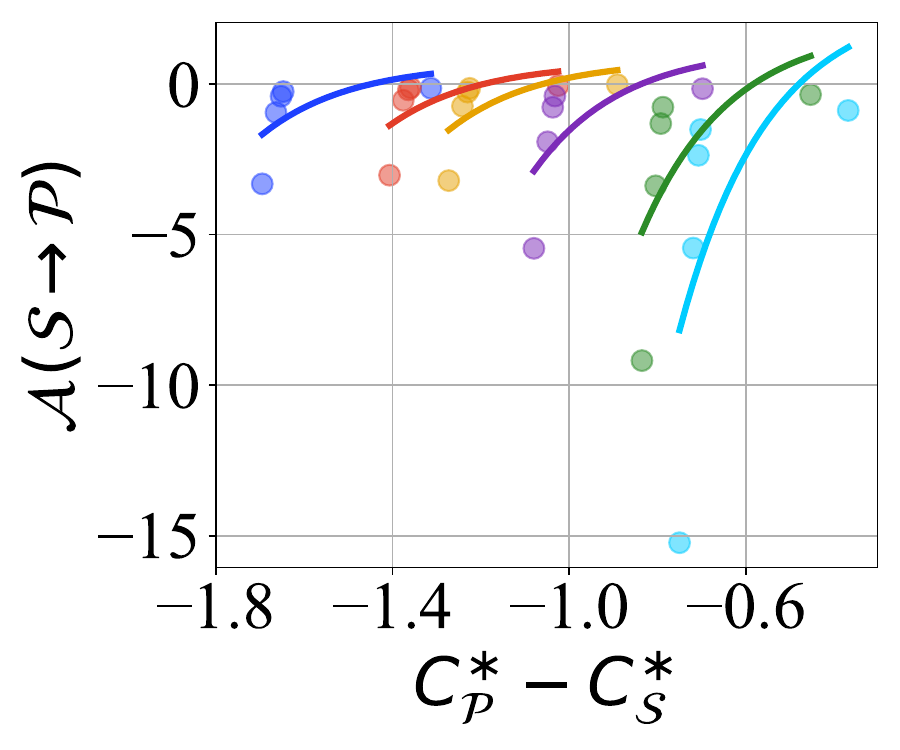}
		\caption{Performance under pruning and parameter sharing.}
		\label{fig:ps}
	\end{minipage}
	\hfill
	\begin{minipage}[t]{0.31\linewidth}
		\centering
		\includegraphics[width=\linewidth]{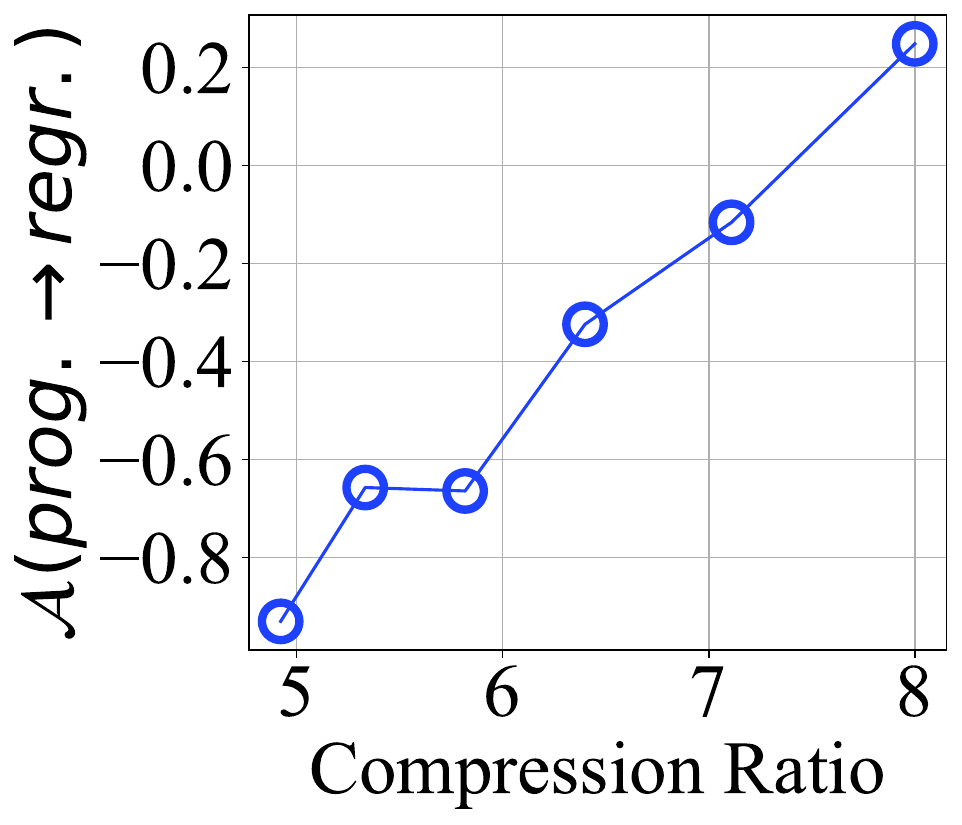}
		\caption{The impact of compression order in MPQ.}
		\label{fig:mpq}
	\end{minipage}
\end{figure}
\vspace{-5mm}

\smallsection{Parameter Efficient Fine-tuning}
Parameter Efficient Fine-Tuning (PEFT), which introduces lightweight low-rank adapters to mitigate compression-induced performance degradation, has recently become a widely adopted practical approach.
We investigate whether the proposed principle remains valid in scenarios where PEFT is applied alongside pruning and quantization.
Figure~\ref{fig:peft} confirms that the same pattern holds for LLaMA 3 8B when combined with SparseGPT ($\pruneonly(\cdot)$), RTN ($\quantonly(\cdot)$), and LoRA~\citep{LoRA} (PEFT).
Applying LoRA after quantization produced a similar corrective effect to rotation, effectively compensating quantization-induced performance loss.
Overall, the progressive intensity hypothesis remains robust under practical training pipelines that include PEFT.


\begin{tcolorbox}[colframe=black!60, colback=lightgrey, boxrule=1pt, arc=3pt, left=2pt, right=2pt, top=2pt, bottom=2pt]
	\textbf{Finding 7.}
	PEFT preserves the hypothesis that stronger compression should be applied later, as post-quantization LoRA effectively restores accuracy and maintains the expected ordering.
\end{tcolorbox}


\vspace{1mm}
\smallsection{Parameter Sharing}
Beyond pruning and quantization, parameter sharing $\mathcal{S}(\cdot)$ is an independent compression technique that ties multiple layers into a unified set of weight parameters.
We conduct joint compression experiments with pruning and parameter sharing to assess whether the principle generalizes beyond the pruning–quantization setting.
In Figure~\ref{fig:ps}, results on LLaMA 2 7B model with Basis Sharing~\citep{BasisSharing} ($\mathcal{S}(\cdot)$) and magnitude-based pruning~\citep{Magnitude} ($\pruneonly(\cdot)$) confirm that the same ordering effect emerges as well.


\begin{tcolorbox}[colframe=black!60, colback=lightgrey, boxrule=1pt, arc=3pt, left=2pt, right=2pt, top=2pt, bottom=2pt]
	\textbf{Finding 8.}
	Joint compression with parameter sharing also follows the hypothesis: placing the stronger operation later yields better performance.
\end{tcolorbox}

%
\vspace{1mm}
\smallsection{Mixed-precision Quantization}
As discussed in Section~\ref{sec:theory}, Mixed Precision Quantization (MPQ) can be formulated as a joint compression problem where each bit-width allocation acts as a separate compression method, satisfying disjoint selectivity.
Figure~\ref{fig:mpq} illustrates the effect of compression order in MPQ, where we sequentially allocate quantization bit-widths following HAWQ-V2~\citep{HAWQ-V2} on ResNet-18 model under a fixed average bit-width (i.e., identical overall compression ratio), comparing progressive (prog.; 8→2 bits) and regressive (regr.; 2→8 bits) sequential allocations.
As the total compression ratio increases, progressive allocation increasingly outperforms regressive allocation in terms of $\coa{\text{prog.}\rightarrow\text{regr.}}$, supporting our hypothesis under MPQ settings.


\begin{tcolorbox}[colframe=black!60, colback=lightgrey, boxrule=1pt, arc=3pt, left=2pt, right=2pt, top=2pt, bottom=2pt]
	\textbf{Finding 9.}
	The Progressive Intensity Hypothesis also holds in MPQ, with progressive bit allocation outperforming regressive allocation since lower-bit quantization acts as stronger compression.
\end{tcolorbox}


\section{Conclusion}
\label{sec:conclusion}
We address the under-explored problem of joint compression order optimization and provide both theoretical and experimental evidences.
The Progressive Intensity Hypothesis (Hypothesis~\ref{hypothesis}) offers a simple yet powerful rule: weaker perturbations first, stronger ones later.
Future works include investigating interference in more complex pipelines, providing explicit predictive rules, and automating compression order selection (see Appendix~\ref{app:subsec:remarks}).

\clearpage
 \section*{Acknowledgments}
 
This work was supported by Institute of Information \& communications Technology Planning \& Evaluation (IITP) grant funded by the Korea government (MSIT)
[No.RS-2020-II200894, Flexible and Efficient Model Compression Method for Various Applications and Environments],
[No.RS-2021-II211343, Artificial Intelligence Graduate School Program (Seoul National University)],
[No.RS-2024-00509257, Global AI Frontier Lab],
and [No.RS-2025-25442338, AI star Fellowship Support Program(Seoul National University)].
This work was supported by Youlchon Foundation.
The Institute of Engineering Research at Seoul National University provided research facilities for this work.
The ICT at Seoul National University provided research facilities for this study.
U Kang is the corresponding author. 

\bibliography{main}

@inproceedings{LampQ,
	title={LampQ: Towards Accurate Layer-wise Mixed Precision Quantization for Vision Transformers},
	author={Kim, Minjun and Lee, Jaeri and Kim, Jongjin and Yun, Jeongin and Kwon, Yongmo and Kang, U},
	booktitle={AAAI},
	year={2026}
}

@article{Falcon,
  author       = {Jun{-}Gi Jang and
                  Chun Quan and
                  Hyun Dong Lee and
                  U Kang},
  title        = {Falcon: lightweight and accurate convolution based on depthwise separable
                  convolution},
  journal      = {Knowl. Inf. Syst.},
  volume       = {65},
  number       = {5},
  pages        = {2225--2249},
  year         = {2023},
}

@article{Pea-KD,
    author = {Cho, Ikhyun AND Kang, U},
    journal = {PLOS ONE},
    publisher = {Public Library of Science},
    title = {Pea-KD: Parameter-efficient and accurate Knowledge Distillation on BERT},
    year = {2022},
    month = {02},
    volume = {17},
    pages = {1-12},
    number = {2},
}

@article{MustaD,
    author = {Kim, Junghun AND Jung, Jinhong AND Kang, U.},
    journal = {PLOS ONE},
    publisher = {Public Library of Science},
    title = {Compressing deep graph convolution network with multi-staged knowledge distillation},
    year = {2021},
    month = {08},
    volume = {16},
    pages = {1-18},
    number = {8},
}

@article{SongHanSurvey,
        author={Deng, Lei and Li, Guoqi and Han, Song and Shi, Luping and Xie, Yuan},
        journal={Proceedings of the IEEE},
        title={Model Compression and Hardware Acceleration for Neural Networks: A Comprehensive Survey},
        year={2020},
        volume={108},
        number={4},
        pages={485-532},
}

@incollection{GholamiSurvey,
        title={A survey of quantization methods for efficient neural network inference},
        author={Gholami, Amir and Kim, Sehoon and Dong, Zhen and Yao, Zhewei and Mahoney, Michael W and Keutzer, Kurt},
        booktitle={Low-Power Computer Vision},
        pages={291-326},
        year={2022},
        publisher={Chapman and Hall/CRC}
}

@inproceedings{KPrune,
        title={Accurate Retraining-free Pruning for Pretrained Encoder-based Language Models},
        author={Park, Seungcheol and Choi, Hojun and Kang, U},
        booktitle={ICLR},
        year={2024}
}

@inproceedings{ImageNet,
        title = {ImageNet: A Large-Scale Hierarchical Image Database},
        author = {Deng, J. and Dong, W. and Socher, R. and Li, L.-J. and Li, K. and Fei-Fei, L.},
        booktitle = {CVPR},
        year = {2009}
}

@inproceedings{ResNet,
        title={Deep residual learning for image recognition},
        author={He, Kaiming and Zhang, Xiangyu and Ren, Shaoqing and Sun, Jian},
        booktitle={CVPR},
        year={2016}
}

@inproceedings{RTN,
        title={Deep learning with limited numerical precision},
        author={Gupta, Suyog and Agrawal, Ankur and Gopalakrishnan, Kailash and Narayanan, Pritish},
        booktitle={ICML},
        year={2015},
}

@inproceedings{BRECQ,
        title={BRECQ: Pushing the Limit of Post-Training Quantization by Block Reconstruction},
        author={Li, Yuhang and Gong, Ruihao and Tan, Xu and Yang, Yang and Hu, Peng and Zhang, Qi and Yu, Fengwei and Wang, Wei and Gu, Shi},
        booktitle={ICLR},
        year={2021}
}

@inproceedings{OPTQ,
        author       = {Frantar, Elias and Ashkboos, Saleh and Hoefler, Torsten and Alistarh, Dan},
        title        = {{OPTQ:} Accurate Quantization for Generative Pre-trained Transformers},
        booktitle    = {ICLR},
        year         = {2023},
}

@inproceedings{DeiT,
	title={Training data-efficient image transformers \& distillation through attention},
	author={Touvron, Hugo and Cord, Matthieu and Douze, Matthijs and Massa, Francisco and Sablayrolles, Alexandre and J{\'e}gou, Herv{\'e}},
	booktitle={ICML},
	year={2021},
}

@inproceedings{RepQViT,
	title={Repq-vit: Scale reparameterization for post-training quantization of vision transformers},
	author={Li, Zhikai and Xiao, Junrui and Yang, Lianwei and Gu, Qingyi},
	booktitle={ICCV},
	year={2023}
}

@article{SensiMix,
	author = {Piao, Tairen AND Cho, Ikhyun AND Kang, U},
	journal = {PLOS ONE},
	publisher = {Public Library of Science},
	title = {SensiMix: Sensitivity-Aware 8-bit index \& 1-bit value mixed precision quantization for BERT compression},
	year = {2022},
	volume = {17},
	pages = {1-22},
	number = {4},
}

@article{PET,
	author = {Jeon, Hyojin AND Park, Seungcheol AND Kim, Jin-Gee AND Kang, U},
	journal = {PLOS ONE},
	publisher = {Public Library of Science},
	title = {PET: Parameter-efficient Knowledge Distillation on Transformer},
	year = {2023},
	month = {07},
	volume = {18},
	pages = {1-21},
	number = {7},	
}

@inproceedings{SynQ,
    title={{SynQ}: Accurate Zero-shot Quantization by Synthesis-aware Fine-tuning},
    author={Kim, Minjun and Kim, Jongjin and Kang, U},
    booktitle={ICLR},
    year={2025}
}

@inproceedings{Interplay,
title={Effective Interplay between Sparsity and Quantization: From Theory to Practice},
author={Simla Burcu Harma and Ayan Chakraborty and Elizaveta Kostenok and Danila Mishin and Dongho Ha and Babak Falsafi and Martin Jaggi and Ming Liu and Yunho Oh and Suvinay Subramanian and Amir Yazdanbakhsh},
booktitle={ICLR},
year={2025},
}

@inproceedings{SparseGPT,
	title = {{S}parse{GPT}: Massive Language Models Can be Accurately Pruned in One-Shot},
	author =  {Frantar, Elias and Alistarh, Dan},
	booktitle = {ICML},
	year = 	 {2023},
}

@inproceedings{QuaRot,
	title={Quarot: Outlier-free 4-bit inference in rotated llms},
	author={Ashkboos, Saleh and Mohtashami, Amirkeivan and Croci, Maximilian L and Li, Bo and Cameron, Pashmina and Jaggi, Martin and Alistarh, Dan and Hoefler, Torsten and Hensman, James},
	booktitle={NeurIPS},
	year={2024}
}

@inproceedings{SLEB,
	title={Sleb: Streamlining llms through redundancy verification and elimination of transformer blocks},
	author={Song, Jiwon and Oh, Kyungseok and Kim, Taesu and Kim, Hyungjun and Kim, Yulhwa and Kim, Jae-Joon},
	booktitle={ICML},
	year={2024}
}

@article{MCforLLMSurvey,
	title={A survey on model compression for large language models},
	author={Zhu, Xunyu and Li, Jian and Liu, Yong and Ma, Can and Wang, Weiping},
	journal={Transactions of the Association for Computational Linguistics},
	volume={12},
	pages={1556--1577},
	year={2024},
}

@inproceedings{SVDQuant,
	title={{SVDQ}uant: Absorbing Outliers by Low-Rank Component for 4-Bit Diffusion Models},
	author={Muyang Li and Yujun Lin and Zhekai Zhang and Tianle Cai and Junxian Guo and Xiuyu Li and Enze Xie and Chenlin Meng and Jun-Yan Zhu and Song Han},
	booktitle={ICLR},
	year={2025},
}

@inproceedings{SVD-LLM,
	title={{SVD}-{LLM}: Truncation-aware Singular Value Decomposition for Large Language Model Compression},
	author={Xin Wang and Yu Zheng and Zhongwei Wan and Mi Zhang},
	booktitle={ICLR},
	year={2025},
}

@inproceedings{BoostViT,
	title={Boost vision transformer with gpu-friendly sparsity and quantization},
	author={Yu, Chong and Chen, Tao and Gan, Zhongxue and Fan, Jiayuan},
	booktitle={CVPR},
	year={2023}
}

@inproceedings{Automatic,
	title={Automatic joint structured pruning and quantization for efficient neural network training and compression},
	author={Qu, Xiaoyi and Aponte, David and Banbury, Colby and Robinson, Daniel P and Ding, Tianyu and Koishida, Kazuhito and Zharkov, Ilya and Chen, Tianyi},
	booktitle={CVPR},
	year={2025}
}

@inproceedings{DeepComp,
	title={Deep compression of pre-trained transformer models},
	author={Wang, Naigang and Liu, Chi-Chun Charlie and Venkataramani, Swagath and Sen, Sanchari and Chen, Chia-Yu and El Maghraoui, Kaoutar and Srinivasan, Vijayalakshmi Viji and Chang, Leland},
	booktitle={NeurIPS},
	year={2022}
}

@inproceedings{AdaptiveQandP,
	title={Adaptive quantization and pruning of deep neural networks via layer importance estimation},
	author={Shinde, Tushar},
	booktitle={Workshop on Machine Learning and Compression, NeurIPS 2024},
	year={2024}
}

@inproceedings{UnderstandInt4,
	title={Understanding int4 quantization for language models: latency speedup, composability, and failure cases},
	author={Wu, Xiaoxia and Li, Cheng and Aminabadi, Reza Yazdani and Yao, Zhewei and He, Yuxiong},
	booktitle={ICML},
	year={2023},
}

@article{PQSurvey,
	title={Pruning and quantization for deep neural network acceleration: A survey},
	author={Liang, Tailin and Glossner, John and Wang, Lei and Shi, Shaobo and Zhang, Xiaotong},
	journal={Neurocomputing},
	volume={461},
	pages={370--403},
	year={2021},
	publisher={Elsevier}
}

@article{PsandQs,
	title={Ps and qs: Quantization-aware pruning for efficient low latency neural network inference},
	author={Hawks, Benjamin and Duarte, Javier and Fraser, Nicholas J and Pappalardo, Alessandro and Tran, Nhan and Umuroglu, Yaman},
	journal={Frontiers in Artificial Intelligence},
	volume={4},
	pages={676564},
	year={2021},
	publisher={Frontiers Media SA}
}

@inproceedings{PruneVSQuant,
	title={Pruning vs quantization: Which is better?},
	author={Kuzmin, Andrey and Nagel, Markus and Van Baalen, Mart and Behboodi, Arash and Blankevoort, Tijmen},
	booktitle={NeurIPS},
	year={2023}
}

@inproceedings{APQ,
	title={Apq: Joint search for network architecture, pruning and quantization policy},
	author={Wang, Tianzhe and Wang, Kuan and Cai, Han and Lin, Ji and Liu, Zhijian and Wang, Hanrui and Lin, Yujun and Han, Song},
	booktitle={CVPR},
	year={2020}
}

@inproceedings{SmoothQuant,
	title={Smoothquant: Accurate and efficient post-training quantization for large language models},
	author={Xiao, Guangxuan and Lin, Ji and Seznec, Mickael and Wu, Hao and Demouth, Julien and Han, Song},
	booktitle={ICML},
	year={2023},
}

@inproceedings{Dejavu,
	title={Deja vu: Contextual sparsity for efficient llms at inference time},
	author={Liu, Zichang and Wang, Jue and Dao, Tri and Zhou, Tianyi and Yuan, Binhang and Song, Zhao and Shrivastava, Anshumali and Zhang, Ce and Tian, Yuandong and Re, Christopher and others},
	booktitle={ICML},
	year={2023},
}

@inproceedings{oBERT,
	title = {The Optimal {BERT} Surgeon: Scalable and Accurate Second-Order Pruning for Large Language Models},
	author = {Kurtic, Eldar  and
	Campos, Daniel  and
	Nguyen, Tuan  and
	Frantar, Elias  and
	Kurtz, Mark  and
	Fineran, Benjamin  and
	Goin, Michael  and
	Alistarh, Dan},
	booktitle = {EMNLP},
	year = {2022},
}

@article{JointJournal,
	author={Motetti, Beatrice Alessandra and Risso, Matteo and Burrello, Alessio and Macii, Enrico and Poncino, Massimo and Pagliari, Daniele Jahier},
	journal={ IEEE Transactions on Computers },
	title={{Joint Pruning and Channel-Wise Mixed-Precision Quantization for Efficient Deep Neural Networks}},
	year={2024},
	volume={73},
	number={11},
	pages={2619-2633},
	publisher={IEEE Computer Society},
}

@inproceedings{PQK,
	title={PQK: Model Compression via Pruning, Quantization, and Knowledge Distillation},
	author={Kim, Jangho and Chang, Simyung and Kwak, Nojun},
	booktitle={Interspeech},
	year={2021}
}

@inproceedings{AccelFPGAs,
	author={Huang, Sitao and Pearson, Carl and Nagi, Rakesh and Xiong, Jinjun and Chen, Deming and Hwu, Wen-mei},
	booktitle={HPEC}, 
	title={Accelerating Sparse Deep Neural Networks on FPGAs}, 
	year={2019},
}

@article{InferenceSurvey,
	title={A survey of techniques for optimizing transformer inference},
	author={Chitty-Venkata, Krishna Teja and Mittal, Sparsh and Emani, Murali and Vishwanath, Venkatram and Somani, Arun K},
	journal={Journal of Systems Architecture},
	volume={144},
	pages={102990},
	year={2023},
	publisher={Elsevier}
}

@inproceedings{OPQ,
	title={Opq: Compressing deep neural networks with one-shot pruning-quantization},
	author={Hu, Peng and Peng, Xi and Zhu, Hongyuan and Aly, Mohamed M Sabry and Lin, Jie},
	booktitle={AAAI},
	year={2021}
}

@inproceedings{ZSQSurvey,
	title={Zero-shot Quantization: A Comprehensive Survey},
	author={Kim, Minjun and Choi, Jaehyeon and Lee, Jongkeun and Cho, Wonjin and Kang, U},
	booktitle={IJCAI},
	year={2025}
}

@inproceedings{McAuleySurvey,
	title={A survey on model compression and acceleration for pretrained language models},
	author={Xu, Canwen and McAuley, Julian},
	booktitle={AAAI},
	year={2023}
}

@article{MLCompSurvey,
	title={A comprehensive review of model compression techniques in machine learning},
	author={Dantas, Pierre Vilar and Sabino da Silva Jr, Waldir and Cordeiro, Lucas Carvalho and Carvalho, Celso Barbosa},
	journal={Applied Intelligence},
	volume={54},
	number={22},
	pages={11804--11844},
	year={2024},
	publisher={Springer}
}

@article{PPFSurvey,
	title={A survey of model compression techniques: Past, present, and future},
	author={Liu, Defu and Zhu, Yixiao and Liu, Zhe and Liu, Yi and Han, Changlin and Tian, Jinkai and Li, Ruihao and Yi, Wei},
	journal={Frontiers in Robotics and AI},
	volume={12},
	pages={1518965},
	year={2025},
	publisher={Frontiers Media SA}
}

@inproceedings{SliceGPT,
	title={SliceGPT: Compress Large Language Models by Deleting Rows and Columns},
	author={Ashkboos, Saleh and Croci, Maximilian L and do Nascimento, Marcelo Gennari and Hoefler, Torsten and Hensman, James},
	booktitle={ICLR},
	year={2024}
}

@inproceedings{SPRINT,
	title={Accurate sublayer pruning for large language models by exploiting latency and tunability information},
	author={Park, Seungcheol and Lee, Sojin and Kim, Jongjin and Lee, Jinsik and Jo, Hyunjik and Kang, U},
	booktitle={IJCAI},
	year={2025}
}

@inproceedings{UniQuanF,
	title={Unifying Uniform and Binary-coding Quantization for Accurate Compression of Large Language Models},
	author={Park, Seungcheol and Bae, Jeongin and Kwon, Beomseok and Kim, Minjun and Kim, Byeongwook and Kwon, Se Jung and Kang, U and Lee, Dongsoo},
	booktitle={ACL},
	year={2025},
}

@inproceedings{KCM,
	title={Gradient-free structured pruning with unlabeled data},
	author={Nova, Azade and Dai, Hanjun and Schuurmans, Dale},
	booktitle={ICML},
	year={2023},
}

@inproceedings{DuQuant,
	title={Duquant: Distributing outliers via dual transformation makes stronger quantized llms},
	author={Lin, Haokun and Xu, Haobo and Wu, Yichen and Cui, Jingzhi and Zhang, Yingtao and Mou, Linzhan and Song, Linqi and Sun, Zhenan and Wei, Ying},
	booktitle={NeurIPS},
	year={2024}
}

@inproceedings{QuipSharp,
	title = {QuIP\string#: Even Better LLM Quantization with Hadamard Incoherence and Lattice Codebooks},
	author = {Tseng, Albert and Chee, Jerry and Sun, Qingyao and Kuleshov, Volodymyr and De Sa, Christopher},
	booktitle = {ICML},
	year = {2024}
}

@inproceedings{SpinQuant,
	title={SpinQuant: LLM Quantization with Learned Rotations},
	author={Liu, Zechun and Zhao, Changsheng and Fedorov, Igor and Soran, Bilge and Choudhary, Dhruv and Krishnamoorthi, Raghuraman and Chandra, Vikas and Tian, Yuandong and Blankevoort, Tijmen},
	booktitle={ICLR},
	year={2025}
}

@inproceedings{LPViT,
	title={Lpvit: Low-power semi-structured pruning for vision transformers},
	author={Xu, Kaixin and Wang, Zhe and Chen, Chunyun and Geng, Xue and Lin, Jie and Yang, Xulei and Wu, Min and Li, Xiaoli and Lin, Weisi},
	booktitle={ECCV},
	year={2024},
}

@inproceedings{GanQ,
	title={{GANQ}: {GPU}-Adaptive Non-Uniform Quantization for Large Language Models},
	author={Pengxiang Zhao and Xiaoming Yuan},
	booktitle={ICML},
	year={2025},
}

@inproceedings{Wanda,
	title={A Simple and Effective Pruning Approach for Large Language Models},
	author={Mingjie Sun and Zhuang Liu and Anna Bair and J Zico Kolter},
	booktitle={ICLR},
	year={2024},
}

@inproceedings{OWQ,
	title={Owq: Outlier-aware weight quantization for efficient fine-tuning and inference of large language models},
	author={Lee, Changhun and Jin, Jungyu and Kim, Taesu and Kim, Hyungjun and Park, Eunhyeok},
	booktitle={AAAI},
	year={2024}
}

@inproceedings{BasisSharing,
	title={Basis Sharing: Cross-Layer Parameter Sharing for Large Language Model Compression},
	author={Jingcun Wang and Yu-Guang Chen and Ing-Chao Lin and Bing Li and Grace Li Zhang},
	booktitle={ICLR},
	year={2025},
}

@inproceedings{Roast,
	title={In defense of parameter sharing for model-compression},
	author={Aditya Desai and Anshumali Shrivastava},
	booktitle={ICLR},
	year={2024},
}

@article{LLaMA2,
	title={Llama 2: Open foundation and fine-tuned chat models},
	author={Touvron, Hugo and Martin, Louis and Stone, Kevin and Albert, Peter and Almahairi, Amjad and Babaei, Yasmine and Bashlykov, Nikolay and Batra, Soumya and Bhargava, Prajjwal and Bhosale, Shruti and others},
	journal={arXiv preprint arXiv:2307.09288},
	year={2023}
}

@article{LLaMA3,
	title={The llama 3 herd of models},
	author={Grattafiori, Aaron and Dubey, Abhimanyu and Jauhri, Abhinav and Pandey, Abhinav and Kadian, Abhishek and Al-Dahle, Ahmad and Letman, Aiesha and Mathur, Akhil and Schelten, Alan and Vaughan, Alex and others},
	journal={arXiv preprint arXiv:2407.21783},
	year={2024}
}

@inproceedings{Magnitude,
	title={Learning both weights and connections for efficient neural network},
	author={Han, Song and Pool, Jeff and Tran, John and Dally, William},
	booktitle={NeurIPS},
	year={2015}
}

@inproceedings{SAViT,
	title={{SAV}iT: Structure-Aware Vision Transformer Pruning via Collaborative Optimization},
	author={Zheng Chuanyang and Zheyang Li and Kai Zhang and Zhi Yang and Wenming Tan and Jun Xiao and Ye Ren and Shiliang Pu},
	booktitle={NeurIPS},
	year={2022},
}

@inproceedings{WikiText2,
	title={Pointer Sentinel Mixture Models},
	author={Stephen Merity and Caiming Xiong and James Bradbury and Richard Socher},
	booktitle={ICLR},
	year={2017},
}

@article{C4,
	title={Exploring the limits of transfer learning with a unified text-to-text transformer},
	author={Raffel, Colin and Shazeer, Noam and Roberts, Adam and Lee, Katherine and Narang, Sharan and Matena, Michael and Zhou, Yanqi and Li, Wei and Liu, Peter J},
	journal={Journal of machine learning research},
	volume={21},
	number={140},
	pages={1--67},
	year={2020}
}

@inproceedings{PRACTISE,
	title={Practical network acceleration with tiny sets},
	author={Wang, Guo-Hua and Wu, Jianxin},
	booktitle={CVPR},
	year={2023}
}

@inproceedings{N2UQ,
	title={Nonuniform-to-uniform quantization: Towards accurate quantization via generalized straight-through estimation},
	author={Liu, Zechun and Cheng, Kwang-Ting and Huang, Dong and Xing, Eric P and Shen, Zhiqiang},
	booktitle={CVPR},
	year={2022}
}

@Misc{accelerate,
	title =        {Accelerate: Training and inference at scale made simple, efficient and adaptable.},
	author =       {Sylvain Gugger and Lysandre Debut and Thomas Wolf and Philipp Schmid and Zachary Mueller and Sourab Mangrulkar and Marc Sun and Benjamin Bossan},
	howpublished = {\url{https://github.com/huggingface/accelerate}},
	year = {2022}
}

@inproceedings{Transformers,
	title = {Transformers: State-of-the-Art Natural Language Processing},
	author = {Thomas Wolf and Lysandre Debut and Victor Sanh and Julien Chaumond and Clement Delangue and Anthony Moi and Pierric Cistac and Tim Rault and Rémi Louf and Morgan Funtowicz and Joe Davison and Sam Shleifer and Patrick von Platen and Clara Ma and Yacine Jernite and Julien Plu and Canwen Xu and Teven Le Scao and Sylvain Gugger and Mariama Drame and Quentin Lhoest and Alexander M. Rush},
	booktitle = {EMNLP},
	year = {2020},
}

@inproceedings{HAWQ-V2,
	title={Hawq-v2: Hessian aware trace-weighted quantization of neural networks},
	author={Dong, Zhen and Yao, Zhewei and Arfeen, Daiyaan and Gholami, Amir and Mahoney, Michael W and Keutzer, Kurt},
	booktitle={NeurIPS},
	year={2020}
}

@article{ARC,
	title={Think you have solved question answering? try arc, the ai2 reasoning challenge},
	author={Clark, Peter and Cowhey, Isaac and Etzioni, Oren and Khot, Tushar and Sabharwal, Ashish and Schoenick, Carissa and Tafjord, Oyvind},
	journal={arXiv preprint arXiv:1803.05457},
	year={2018}
}

@inproceedings{HellaSwag,
	title={HellaSwag: Can a Machine Really Finish Your Sentence?},
	author={Zellers, Rowan and Holtzman, Ari and Bisk, Yonatan and Farhadi, Ali and Choi, Yejin},
	booktitle ={ACL},
	year={2019}
}

@inproceedings{LAMBADA,
	title={The LAMBADA dataset: Word prediction requiring a broad discourse context},
	author={Paperno, D and Kruszewski, G and Lazaridou, A and Pham, QN and Bernardi, Raffaella and Pezzelle, S and Baroni, M and Boleda, G and Fern{\'a}ndez, R},
	booktitle={ACL},
	year={2016},
}

@inproceedings{PIQA,
	title={Piqa: Reasoning about physical commonsense in natural language},
	author={Bisk, Yonatan and Zellers, Rowan and Gao, Jianfeng and Choi, Yejin and others},
	booktitle={AAAI},
	year={2020}
}

@article{Winogrande,
	title={Winogrande: An adversarial winograd schema challenge at scale},
	author={Sakaguchi, Keisuke and Bras, Ronan Le and Bhagavatula, Chandra and Choi, Yejin},
	journal={Communications of the ACM},
	volume={64},
	number={9},
	pages={99--106},
	year={2021},
}

@inproceedings{BERT,
	title={Bert: Pre-training of deep bidirectional transformers for language understanding},
	author={Devlin, Jacob and Chang, Ming-Wei and Lee, Kenton and Toutanova, Kristina},
	booktitle={NAACL},
	year={2019}
}

@article{Mistral,
	title={Mistral 7B}, 
	author={Albert Q. Jiang and Alexandre Sablayrolles and Arthur Mensch and Chris Bamford and Devendra Singh Chaplot and Diego de las Casas and Florian Bressand and others},
	year={2023},
	journal={arXiv preprint arXiv:2310.06825},
}

@inproceedings{STS-B,
	title={SemEval-2017 Task 1: Semantic Textual Similarity Multilingual and Crosslingual Focused Evaluation},
	author={Cer, Daniel and Diab, Mona and Agirre, Eneko and Lopez-Gazpio, I{\~n}igo and Specia, Lucia},
	booktitle={Eleventh International Workshop on Semantic Evaluation},
	year={2017}
}

@inproceedings{GLUE,
	title={{GLUE}: A Multi-Task Benchmark and Analysis Platform for Natural Language Understanding},
	author={Alex Wang and Amanpreet Singh and Julian Michael and Felix Hill and Omer Levy and Samuel R. Bowman},
	booktitle={ICLR},
	year={2019},
}

@misc{Mistral_Nemo,
	title={{Mistral NeMo: our new best small model}},
	author={{Mistral AI Team}},
	howpublished={\url{https://mistral.ai/news/mistral-nemo}},
	year={2024},
	month={July},
	note={Accessed: 2025}
}

@inproceedings{
	LongPPL,
	title={What is Wrong with Perplexity for Long-context Language Modeling?},
	author={Lizhe Fang and Yifei Wang and Zhaoyang Liu and Chenheng Zhang and Stefanie Jegelka and Jinyang Gao and Bolin Ding and Yisen Wang},
	booktitle={ICLR},
	year={2025},
}

@inproceedings{BeyondPPL,
	title = {Language Model Evaluation Beyond Perplexity},
	author = {Meister, Clara  and
	Cotterell, Ryan},
	booktitle = {ACL},
	year = {2021},
}

@inproceedings{LoRA,
	title={Lo{RA}: Low-Rank Adaptation of Large Language Models},
	author={Edward J Hu and yelong shen and Phillip Wallis and Zeyuan Allen-Zhu and Yuanzhi Li and Shean Wang and Lu Wang and Weizhu Chen},
	booktitle={ICLR},
	year={2022},
}

@Misc{PEFT,
	title =        {{PEFT}: State-of-the-art Parameter-Efficient Fine-Tuning methods},
	author =       {Sourab Mangrulkar and Sylvain Gugger and Lysandre Debut and Younes Belkada and Sayak Paul and Benjamin Bossan},
	howpublished = {\url{https://github.com/huggingface/peft}},
	year =         {2022}
}
\bibliographystyle{iclr2026_conference}

\appendix
\newpage

\section{Notation}
\label{app:sec:notation}
We summarize the frequently used notations in the paper as Table~\ref{app:tab:notation}.

\def\arraystretch{1.0}
\begin{table}[ht]
	\centering
	\caption{Frequently used notations.}
    \label{app:tab:notation}
		\begin{tabular}{cl}
		\hline
		\toprule
		\textbf{Symbol} & \textbf{Description} \\
		
		\midrule
		
		$\model$ 
		& A pre-trained model \\
		
		$\compmodel$ 
		& A compressed model \\
		
		$\comp{\cdot} \in \compset$ 
		& Compression method from a set $\compset$ \\
		
		$\comprate$
		& Compression ratio \\
				
		$\metric{\cdot}$ 
		& Performance metric \\
		
		$\weighti, \acti$ 
		& Weight and activation matrices of layer $\layeri$, respectively \\
		
		$\perm \in \permset$ 
		& Compression order from a set $\permset$ of all possible permutations \\
		
		\midrule
		
		$\prune{\cdot}, \quant{\cdot}$
		& Pruning and quantization methods, respectively \\
		
		$\pruneratio$ 
		& Pruning ratio ($\comprateprune = 1/(1-p)$) \\
		
		$\orgbits, \quantbits$
		& Original and quantized bit-widths, respectively \\
		
		$\erroru{\componly}{\cdot}$
		& Error induced by applying $\comp{\cdot}$ \\
		
		$\type \in \typeset{\model}$
		& Abstract data type from the set of all valid types in model $\model$  \\
		
		$\typegran{\componly}$ 
		& Granularity of $\comp{\cdot}$ \\
		
		$\unit \in \units{\model; \type}$
		& A unit of type $\type$ within the model $\model$ \\
		
		$\applies{\unit}{\componly}{\pi}$
		& Binary indicator of whether $\comp{\cdot}$ modifies unit $\unit$ under order $\pi$ \\
		
		\midrule

 		$\pg{\componeonly, \comptwoonly}$
		 &  Performance gap between $\compone{\cdot}$ and $\comptwo{\cdot}$ \\	 
		 
		 $\coa{\componeonly \rightarrow \comptwoonly}$ 
		 & Compression order advantage of $\componeonly \rightarrow \comptwoonly$ over $\comptwoonly \rightarrow \componeonly$ \\
		 
		 $\compeqrate_{\componly}$ 
		 & Compression Equivalent Ratio (CER) of $\comp{\cdot}$ \\
		 
		 $\interference{\model; \componeonly \rightarrow \comptwoonly}$ 
		 & Interference from $\compone{\cdot}$ to $\comptwo{\cdot}$ \\
		 		
		\bottomrule
		\hline
		\end{tabular}
\end{table}



\section{Details on Theoretical Analysis}
\label{app:sec:theory_detail}
	%
\begin{wrapfigure}[9]{R}{4cm}
	\vspace{-19.2mm}
	\centering
	\includegraphics[width=\linewidth]{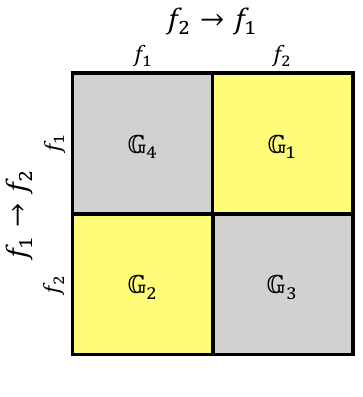}
	\vspace{-8mm}
	\caption{
		We partition all units $\unit$ into four disjoint groups.
		Only groups $\groupone$ and $\grouptwo$ influence $\coa{\componeonly \rightarrow \comptwoonly}$.
	}
	\label{app:fig:proof_case}
\end{wrapfigure}
%
We provide the detailed proofs for Theorems~\ref{thm:1} and~\ref{thm:2}, then formulate a generalized version of Hypothesis~\ref{hypothesis} applicable to a broader setting with multiple compression methods.

\subsection{Proof of Theorem~\ref{thm:1}}
\label{app:subsec:proof_thm1}

\begin{proof}
	Given two compression methods $\compone{\cdot}$ and $\comptwo{\cdot}$ with respective granularities $\typegran{\componeonly}$ and $\typegran{\comptwoonly}$, disjoint selectivity ensures that every unit is assigned exclusively to one method.
    Hence, every unit $\unit \in \units{\model; \lut(\typegran{\componeonly}, \typegran{\comptwoonly}) }$ is classified into one of four disjoint groups, $\groupone, \grouptwo, \groupthree,$ or $\groupfour$, according to its assigned method.
    Then,

    \begin{align*}
    	& \groupone = \{\unit ~|~ \applies{\unit}{\componeonly}{\comptwoonly \circ \componeonly} = 1, ~ \applies{\unit}{\componeonly}{\componeonly \circ \comptwoonly} = 0 \}, \\
    	& \grouptwo = \{\unit ~|~ \applies{\unit}{\componeonly}{\comptwoonly \circ \componeonly} = 0, ~ \applies{\unit}{\componeonly}{\componeonly \circ \comptwoonly} = 1 \}, \\
    	& \groupthree = \{\unit ~|~ \applies{\unit}{\componeonly}{\comptwoonly \circ \componeonly} = 0, ~ \applies{\unit}{\componeonly}{\componeonly \circ \comptwoonly} = 0 \}, \\
    	& \groupfour = \{\unit ~|~ \applies{\unit}{\componeonly}{\comptwoonly \circ \componeonly} = 1, ~ \applies{\unit}{\componeonly}{\componeonly \circ \comptwoonly} = 1 \},\\
    	& \groupone \cup \grouptwo \cup \groupthree \cup \groupfour = \units{\model; \lut(\typegran{\componeonly}, \typegran{\comptwoonly})},
    \end{align*}

    where $\units{\model; \type}$ represents the unit set of model $\model$ at granularity $\type$,
    and $\applies{\unit}{\componly}{\pi}$ records whether $\comp{\cdot}$ modifies $\unit$ under the ordering $\pi$ (1 if yes, 0 if no).
    Note that these four groups are mutually exclusive and collectively exhaustive.
    Also, $|\groupone| = |\grouptwo|$ since compression ratios $ \comprate_{\componeonly}$ and $ \comprate_{\comptwoonly}$ are identical regardless of the compression order.
    Figure~\ref{app:fig:proof_case} illustrates the four groups.

    Under Assumption~\ref{assumption:general} and the defined partitioning of groups, the compression order advantage $\coa{\cdot}$ is expressed in terms of unit-wise reconstruction errors $\erroru{\componly}{\uniti}$:
    \begin{align*}
    	& \coa{\componeonly \rightarrow \comptwoonly} \\
    	&= \metric{(\comptwoonly \circ \componeonly)(\phi)} - \metric{(\componeonly \circ \comptwoonly)(\phi)} \\
    	&= - \beta \big( \erroru{(\comptwoonly \circ \componeonly)}{\model} - \erroru{(\componeonly \circ \comptwoonly)}{\model} \big) \\    	
    	&= - \beta \Big(
    	\sum_{\uniti \in \groupone}{\|\erroru{\componeonly}{\uniti}\|^2_F}
    	+ \sum_{\uniti \in \grouptwo}{\|\erroru{\comptwoonly}{\uniti}\|^2_F}
    	+ \sum_{\uniti \in \groupthree}{\|\erroru{\comptwoonly}{\uniti}\|^2_F}
    	+ \sum_{\uniti \in \groupfour}{\|\erroru{\componeonly}{\uniti}\|^2_F} \\
    	&\quad\quad\quad\,\,\,\, - \sum_{\uniti \in \groupone}{\|\erroru{\comptwoonly}{\uniti}\|^2_F}
    	- \sum_{\uniti \in \grouptwo}{\|\erroru{\componeonly}{\uniti}\|^2_F}
    	- \sum_{\uniti \in \groupthree}{\|\erroru{\comptwoonly}{\uniti}\|^2_F}
    	- \sum_{\uniti \in \groupfour}{\|\erroru{\componeonly}{\uniti}\|^2_F}
    	\Big) \\
    	& = - \beta \Big(
    	\sum_{\uniti \in \groupone}{\|\erroru{\componeonly}{\uniti}\|^2_F}
    	+ \sum_{\uniti \in \grouptwo}{\|\erroru{\comptwoonly}{\uniti}\|^2_F}
    	- \sum_{\uniti \in \groupone}{\|\erroru{\comptwoonly}{\uniti}\|^2_F}
    	- \sum_{\uniti \in \grouptwo}{\|\erroru{\componeonly}{\uniti}\|^2_F}
    	\Big) \\
    	& = \beta \big(
    	\sum_{\uniti \in \grouptwo} {g(\uniti)}
    	- \sum_{\uniti \in \groupone} {g(\uniti)}
    	\big),
    \end{align*}
	where error difference $g(\uniti) = \big\| \erroru{\componeonly}{\uniti} \big\|_F^2 - \big\| \erroru{\comptwoonly}{\uniti} \big\|_F^2$.
	Note that $\groupthree$ and $\groupfour$ are discarded since their effect remains unchanged irrespective of the compression order.

\end{proof}

\smallsection{Case study on pruning and quantization}
To support intuition, we provide a case study on pruning and quantization, which constitute the core scenario of our work.
As described in the main text, disjoint selectivity holds only when the granularity $\typegran{\pruneonly}$ of pruning $\prune{\cdot}$ is greater than or equal to the granularity $\typegran{\quantonly}$ of quantization $\quant{\cdot}$.
We analyze this at the layer level: let $\weighti$ and $\acti$ denote the weight and activation of a layer $\layeri \in \layers$ in the model $\model$.
Note that the error $\erroru{\componly}{\weighti, \acti} = \comp{\weighti} \comp{\acti} - \weighti \acti$ for a compression method $\comp{\cdot}$, as described in Section~\ref{sec:prelim}.

We partition the layers $\layers$ into four disjoint groups $\groupone$, $\grouptwo$, $\groupthree$, and $\groupfour$ based on their pruning status:
\begin{align*}
	& \groupone = \{\unit ~|~ \applies{\unit}{\quantonly}{\pruneonly \circ \quantonly} = 1, ~ \applies{\unit}{\quantonly}{\quantonly \circ \pruneonly} = 0 \}, \\
	& \grouptwo = \{\unit ~|~ \applies{\unit}{\quantonly}{\pruneonly \circ \quantonly} = 0, ~ \applies{\unit}{\quantonly}{\quantonly \circ \pruneonly} = 1 \}, \\
	& \groupthree = \{\unit ~|~ \applies{\unit}{\quantonly}{\pruneonly \circ \quantonly} = 0, ~ \applies{\unit}{\quantonly}{\quantonly \circ \pruneonly} = 0 \}, \\
	& \groupfour = \{\unit ~|~ \applies{\unit}{\quantonly}{\pruneonly \circ \quantonly} = 1, ~ \applies{\unit}{\quantonly}{\quantonly \circ \pruneonly} = 1 \},\\
	& \groupone \cup \grouptwo \cup \groupthree \cup \groupfour = \layers,
\end{align*}
 where $\applies{\unit}{\componly}{\pi}$ records whether $\comp{\cdot}$ modifies $\unit$ under the ordering $\pi$ (1 if yes, 0 if no).
In the pruning–quantization setting, the partition is directly determined by whether each layer is pruned in the final model.
Pruning behaves as an absorbing operator: pruning overrides any modification introduced by quantization.
Therefore, the partition above reduces to grouping layers according to whether they are pruned under $\ptq$ or $\qtp$, yielding the pruning-status-based formulation below:
%
%
\begin{align*}
	& \groupone = \prunedlayers_{\ptq} \setminus \prunedlayers_{\qtp}, \\
	& \grouptwo = \prunedlayers_{\qtp} \setminus \prunedlayers_{\ptq}, \\
	& \groupthree = \prunedlayers_{\ptq} \cap \prunedlayers_{\qtp}, \\
	& \groupfour = \layers \setminus (\prunedlayers_{\ptq} \cup \prunedlayers_{\qtp}), \\
	& \groupone \cup \grouptwo \cup \groupthree \cup \groupfour = \layers,
\end{align*}
where $\prunedlayers_{\componly}$ denote the sets of pruned layers when applying $\comp{\cdot}$.

Then, the quantization-first advantage $\coa{\quantonly \rightarrow \pruneonly}$ is estimated as follows:
\begin{align*}
	& \coa{\quantonly \rightarrow \pruneonly}
	= \metric{(\qtp)(\phi)} - \metric{(\ptq)(\phi)}
	= - \beta \big( \erroru{\qtp}{\model} - \erroru{\ptq}{\model} \big) \\
	&= - \beta \Big( \sum_{\layeri \in \layers}{
		\big\|\errorqtp{\weighti, \acti}\big\|_F^2 - \big\|\errorptq{\weighti, \acti}\big\|_F^2} \Big)\\
	&= - \beta \bigg( \sum_{\layeri \in \groupone} { \Big\{
		{\big\| \errorq{\weighti, \acti} \big\|_F^2 - \big\| -\weighti \acti \big\|_F^2} \Big\}}
	+ \sum_{\layeri \in \grouptwo} { \Big\{
		\big\| -\weighti \acti \big\|_F^2 - \big\| \errorq{\weighti, \acti} \big\|_F^2 \Big\}} \bigg) \\
	&= \beta \bigg( \sum_{\layeri \in \grouptwo} {g(\weighti, \acti)} - \sum_{\layeri \in \groupone} {g(\weighti, \acti)} \bigg),
\end{align*}
where $ g(\weighti, \acti)
= \big\| \errorq{\weighti, \acti} \big\|_F^2 - \big\| -\weighti \acti \big\|_F^2$.
This expression holds as for any layer $\layeri \in \layers$, the pruning operator and its associated error are defined as follows:
\begin{equation*}
	\prune{\weighti} \prune{\acti} =
	\begin{cases}
		\mat{0} & \text{if pruned} \\
		\weighti \acti & \text{otherwise}
	\end{cases},
	\\ \quad
	\errorp{\weighti, \acti} =
	\begin{cases}
		-\weighti \acti & \text{if pruned} \\
		\mat{0} & \text{otherwise}
	\end{cases}.
\end{equation*}

\subsection{Proof of Theorem~\ref{thm:2}}
\label{app:subsec:proof_thm2}
\begin{proof}
	As $\coa{\quantonly \rightarrow \pruneonly}$ and $\pg{\cdot}$ (or $\compeqrateprune - \compratequant$) are functions of the compression ratio $\compratequant$ we analyze their behavior separately.
	Without loss of generality, we consider only the case where  $\compratequant$ changes in the direction of increasing $\compeqrateprune - \compratequant$, i.e., decreasing $\compratequant$.
	
	Under Assumption~\ref{assumption:pandq}, which assumes well-designed quantization, decreasing $\compratequant$ preserves the expected value of the quantized outputs while decreasing their standard deviation.
    As pruning intensity is held constant, the variation across compression orders is attributed solely to the severity of quantization.
    Lower quantization ratio (i.e., smaller standard deviation) decreases the chance that units behave differently across orders, which can only decrease or preserve the value of $|\groupone| = |\grouptwo|$, but never increase it.
    We analyze the two possible cases as follows.

    \begin{itemize}[leftmargin=3.4mm, itemsep=0mm]
    	\item \textbf{Case 1: Number of layers affected by order decreases.}
    	Although more than one layer may change simultaneously, any such change can be decomposed into a sequence in which layers are added one by one; thus it suffices to analyze the case where exactly one layer is added.
		Let $\layeri$ and $\layerj$ denote the layers moving from $\groupone$ and $\grouptwo$ respectively.
		Such a transition occurs in a budget-preserving manner: one layer from $\groupone$ and one from $\grouptwo$ are jointly reallocated, with one moving to $\groupthree$ and the other to $\groupfour$, so that the number of pruned layers remains fixed.
		As only $\groupone$ and $\grouptwo$ contribute to $\coa{\quantonly \rightarrow \pruneonly}$, removing $\layeri \in \groupone$ and $\layerj \in \grouptwo$ from these groups changes the value by $-\beta \big( g(\layerj) - g(\layeri) \big)$ (see Theorem~\ref{thm:1}).
		Therefore, to show that $\coa{\quantonly \rightarrow \pruneonly}$ does not decrease, it suffices to prove that
		$g(\layerj) - g(\layeri) \leq 0.$
		Expanding the definition of $g(\cdot)$, we obtain
		\[
		g(\layerj) - g(\layeri)
		= \big(\|\errorq{\weightj, \actj}\|_F^2 - \|\errorq{\weighti, \acti}\|_F^2\big) - \big(\|-\weightj \actj\|_F^2 - \|-\weighti \acti\|_F^2\big).
		\]
		Under Assumptions~\ref{assumption:general} and \ref{assumption:pandq}, the second term in parentheses is positive, i.e., $\|-\weightj \actj\|_F^2 - \|-\weighti \acti\|_F^2 > 0$.
		This is because under pruning alone, $\layeri \in \groupone$ is pruned while $\layerj \in \grouptwo$ is not.
		Under the well-designed pruning assumption, which minimizes performance drop, this is equivalent to minimizing the increased error.
		Therefore, the pruning error $ |-\weighti \acti|_F^2$ for pruned $\layeri$ is less than or equal to $ |-\weightj \actj|_F^2$ for unpruned $\layerj$, making the term positive.
		
		Given the assumption of well-designed quantization in Assumption~\ref{assumption:pandq}, the remaining first term, which denotes the difference in quantization errors, is negligible compared to the pruning-related component.
		This is because the quantization error at each layer is modeled as zero-mean noise with small variance, whereas the pruning error term $\|-\weighti \acti\|_F^2$ (or $\|-\weightj \actj\|_F^2$) corresponds directly to the magnitude of the pruned responses.
		Thus, the difference $\|\errorq{\weightj, \actj}\|_F^2 - \|\errorq{\weighti, \acti}\|_F^2$ remains uniformly small compared to $\|-\weightj \actj\|_F^2 - \|-\weighti \acti\|_F^2$, so the pruning-induced gap dominates $g(\layerj) - g(\layeri)$.
		Consequently, we get
		\[
		g(\layerj) - g(\layeri)
		\approx - \big(\|-\weightj \actj\|_F^2 - \|-\weighti \acti\|_F^2\big) < 0.
		\]
    	Overall, the decrease in the number of order-dependent layers leads to the increase of the order advantage $\coa{\quantonly \rightarrow \pruneonly}$.
    	
    	\item \textbf{Case 2: Number of layers affected by order remains unchanged.}
    	As the loss-contributing groups $\groupone$ and $\grouptwo$ do not change, the compression order advantage $\coa{\quantonly \rightarrow \pruneonly}$ remains unaffected.

    \end{itemize}
	
    Therefore, in all cases where $\compratequant$ decreases, the value of $\coa{\quantonly \rightarrow \pruneonly}$ does not decrease.
	In conclusion, monotonicity between $\coa{\quantonly \rightarrow \pruneonly}$ and $\pg{\cdot}$ holds under fixed $\comprateprune$.
\end{proof}

\subsection{Generalization to multiple methods}
\label{app:subsec:theory_generalize}
We formulate Hypothesis~\ref{hypothesis} with two compression methods $\compone{\cdot}$ and $\comptwo{\cdot}$.
This is because if the hypothesis holds for any pair of methods, it can be generalized to more than two methods.

Following the setup in Problem~\ref{problem}, suppose we sequentially apply a set $\compset = {\compone{\cdot}, \comptwo{\cdot}, \cdots, \compn{\cdot}}$ of compression methods to a pre-trained model $\model$.
Then, any pair $(\perm_1, \perm_2)$ of permutations from the set $\permset= \{ \perm: \compset \rightarrow \compset ~|~ \perm \text{ is bijective}\}$ of all permutations can be converted into one another via a sequence of adjacent transpositions.
This is because the adjacent transpositions generate the full symmetric group, allowing any permutation to be constructed from another.
Thus, under Hypothesis~\ref{hypothesis}, our original claim shown in Figure~\ref{fig:hypothesis} holds, that applying stronger permutations later leads to better performance of the compressed model. 

\section{Experimental Setup}
\label{app:sec:setup}

\def\arraystretch{0.8}
\begin{table}[t]
	\centering
	\caption{Baseline methods covered in our experiments across different settings.}
	\label{app:tab:baselines}
	\setlength{\tabcolsep}{10pt}
	\begin{tabular}{cccc}
		\hline
		\toprule
		\textbf{Compression} &
		\textbf{Modality} &  \textbf{Target Models} & \textbf{Baseline Methods} \\
		\midrule
		\multirow{4}[11]{*}{\makecell{Pruning \\ $\prune{\cdot}$}} & \multirow{2}[7]{*}{\makecell{Language \\ models}} & \makecell{Decoder-only \\ models} & \makecell{SparseGPT~\citep{SparseGPT}, \\ Wanda~\citep{Wanda}, \\ SLEB~\citep{SLEB}} \\
		\cmidrule{3-4}
		& & \makecell{Encoder-based \\ models} &
		\makecell{K-prune~\citep{KPrune}} \\
		\cmidrule{2-4}
		& \multirow{2}[3]{*}{\makecell{Vision \\ models}\vspace{1mm}} & \makecell{CNNs} & \makecell{PRACTISE~\citep{PRACTISE}} \\
		\cmidrule{3-4}
		& & \makecell{ViTs} & \makecell{SAViT~\citep{SAViT}} \\
		\midrule
		\multirow{4}[11]{*}{\makecell{Quantization \\ $\quant{\cdot}$}} & \multirow{2}[7]{*}{\makecell{Language \\ models}} & \makecell{Decoder-only \\ models} & \makecell{RTN~\citep{RTN}, \\ OPTQ~\citep{OPTQ}, \\ QuaRoT~\citep{QuaRot}} \\
		\cmidrule{3-4}
		& & \makecell{Encoder-based \\ models} &
		\makecell{UniQuanF~\citep{UniQuanF}} \\
		\cmidrule{2-4}
		& \multirow{2}[3]{*}{\makecell{Vision \\ models}\vspace{1mm}} & \makecell{CNNs} & \makecell{N2UQ~\citep{N2UQ}} \\
		\cmidrule{3-4}
		& & \makecell{ViTs} & \makecell{RepQ-ViT~\citep{RepQViT}} \\
		\midrule
		\makecell{Parameter \\ Efficient \\ Fine-tuning} & \makecell{Language \\ models} & \makecell{Decoder-only \\ models} & LoRA~\citep{LoRA}\\
		\midrule		
		\makecell{Parameter \\ Sharing} & \makecell{Language \\ models} & \makecell{Decoder-only \\ models} & Basis Sharing~\citep{BasisSharing}\\
		\midrule
		\makecell{Mixed-precision \\ quantization} & \makecell{Vision \\ models} & CNNs & HAWQ-V2~\citep{HAWQ-V2} \\
		\bottomrule
		\hline
	\end{tabular}
	\vspace{-2mm}
\end{table}


We describe the details on the experimental setup, including models, datasets, baselines, evaluation protocol, and implementation.

\smallsection{Models}
We evaluate representative models across modalities, including LLaMA 2 (7B, 13B)~\citep{LLaMA2}, LLaMA 3 8B~\citep{LLaMA3}, Mistral 7B~\citep{Mistral}, Mistral Nemo-12B~\citep{Mistral_Nemo}, and BERT~\citep{BERT} for language, and ResNet-18~\citep{ResNet} and DeiT-Base~\citep{DeiT} for vision.

\smallsection{Datasets}
We evaluate decoder-only language models on WikiText-2~\citep{WikiText2} and C4~\citep{C4} datasets
for perplexity, and on five commonsense reasoning tasks—ARC~\citep{ARC}, HellaSwag~\citep{HellaSwag}, LAMBADA~\citep{LAMBADA}, PIQA~\citep{PIQA}, and Winogrande~\citep{Winogrande}.
For encoder-based models, we evaluate the performance using Spearman's rank correlation coefficient on the STS-B dataset.
For vision models, we report classification accuracy on ImageNet (ILSVRC 2012)~\citep{ImageNet} dataset.

\smallsection{Baseline Methods}
We validate our hypothesis across a total of sixteen pairs of compression methods by incorporating six pruning methods, six quantization methods, and one mixed-precision quantization method.
Table~\ref{app:tab:baselines} provides an overview of baseline methods categorized by target models and modalities.
Refer to the original papers for further details.

\smallsection{Evaluation Protocol}
The calibration dataset consists of a single batch with 128 samples drawn from the same dataset used for perplexity evaluation.
We set the batch size to 16 for perplexity evaluation and to 128 for commonsense reasoning tasks.
All quantization methods apply the same bit-width to weights, activations, and KV-cache, with clipping applied during weight quantization.
Both models are evaluated without fine-tuning using a batch size of 128.
Metrics are reported as the average of five repeated runs, each computed with four-digit precision.
We plot the relative values for visualization.

\smallsection{Implementation and Machine}
Our implementations are written in Python and rely on PyTorch, Transformers, Accelerate, and TorchVision libraries.
For all baseline methods, we reproduce the results based on their open-source code and hyperparameter configurations.
All of our experiments are done at a workstation with Intel Xeon Gold 6338 and NVIDIA A100 80GB.

\smallsection{Parameter Efficient Fine-tuning Experiment}
We adopt LoRA~\citep{LoRA} on top of SparseGPT~\citep{SparseGPT} and RTN~\citep{RTN} to fine-tune the compressed model.
The target model is LLaMA 3 8B~\citep{LLaMA3}, where fine-tuning is processed after quantization.
We select WikiText-2 as the calibration dataset and train for a total of 2 epochs.
We follow Basis Sharing~\citep{BasisSharing} for training details of the low-rank adapter, while exploiting the PEFT~\citep{PEFT} library.

\smallsection{Parameter Sharing Experiment}
We evaluate the performance of LLaMA 2 7B model when applying Basis Sharing~\citep{BasisSharing} and magnitude-based pruning~\citep{Magnitude}.
We follow Basis Sharing for hyperparameters regarding parameter sharing, where the group size is 128.
QuaRot~\citep{QuaRot} is selected as the quantization baseline for calculating compression equivalent bits.
No additional fine-tuning is applied under this setting.

\smallsection{Mixed-precision Quantization Experiment}
We base our method on HAWQ-V2~\citep{HAWQ-V2}, but allocate bit-widths iteratively rather than in a single shot following LampQ~\citep{LampQ}.
At each iteration, we search per-layer bit-widths from a range of [2, 3, 4, 5, 6, 7, 8] and train for 5 epochs.
All other experimental settings, including hyperparameters and quantization techniques, are aligned with the original paper.
We run all MPQ experiments on a workstation with Intel Xeon Silver 4310 and NVIDIA RTX 4090. 

\section{Further Discussion and Experiments}
\label{app:sec:discussion}
We present results from extended experiments, and offer further discussion and remarks on our work.

\subsection{Experiments on Diverse LLMs}
\label{app:subsec:llm_results}
Our experiments on decoder-only models are limited to LLaMA herd models (LLaMA 2~\citep{LLaMA2} and LLaMA 3~\citep{LLaMA3} models), which may not fully reflect broader generality.
To address this, we conduct additional experiments on models from the Mistral herd.
Figure~\ref{app:fig:mistral} presents the results of applying SparseGPT ($\prune{\cdot}$) and QuaRot ($\quant{\cdot}$) to Mistral 7B~\citep{Mistral} and Mistral Nemo 12B~\citep{Mistral_Nemo}.
We have two observations from the result.
First, the compression-order trend aligns well with the hypothesis across Mistral-based models.
The result serves as additional evidence confirming the hypothesis in decoder-only language models.
Second, comparing models within the same herd (also refer to Figure~\ref{fig:llm_result}), we find that smaller models exhibit greater variation in compression-order advantage for identical CER differences.
This may be because low-bit quantization (or stronger quantization) causes greater degradation in smaller models, thereby intensifying the observed differences.

\begin{figure}[t]
	\centering
	\includegraphics[width=0.9\linewidth]{main_result_legend}\\
	\vspace{2mm}
	\hspace{0.01\linewidth}
	\begin{subfigure}[t]{0.34\linewidth}
		\centering
		\includegraphics[width=\linewidth]{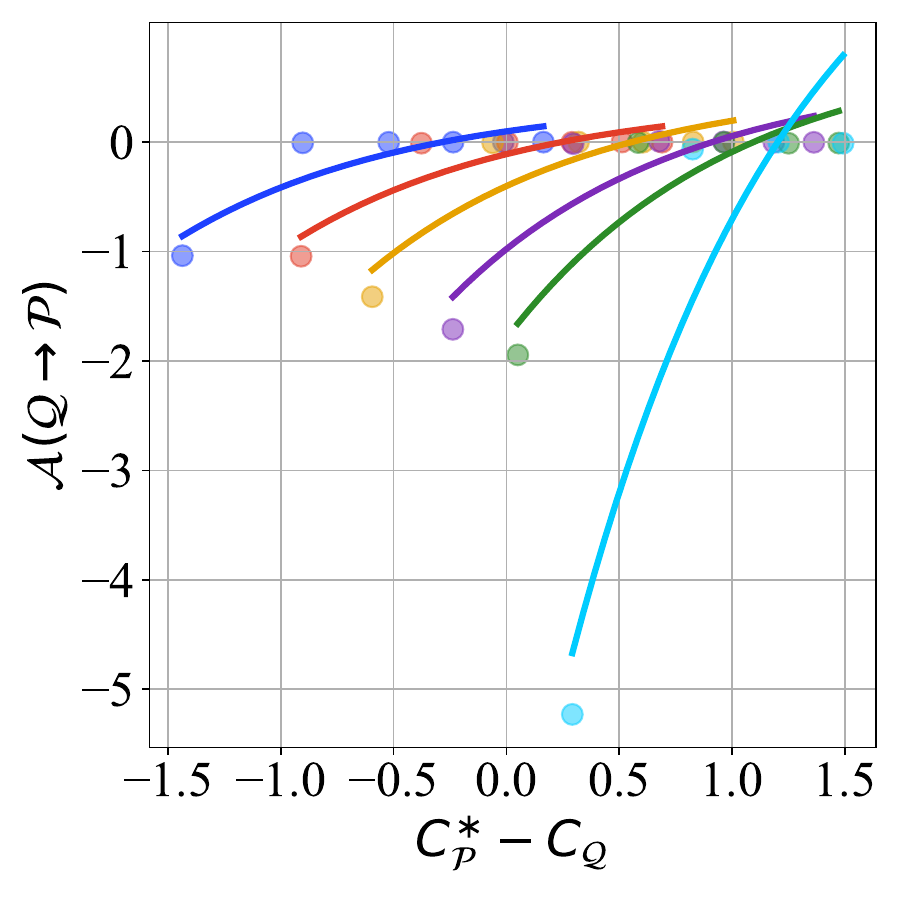}
		\caption{Mistral 7B}
		\label{app:fig:mistral:7b}
	\end{subfigure}
	\hspace{15mm}%
	\begin{subfigure}[t]{0.34\linewidth}
		\centering
		\includegraphics[width=\linewidth]{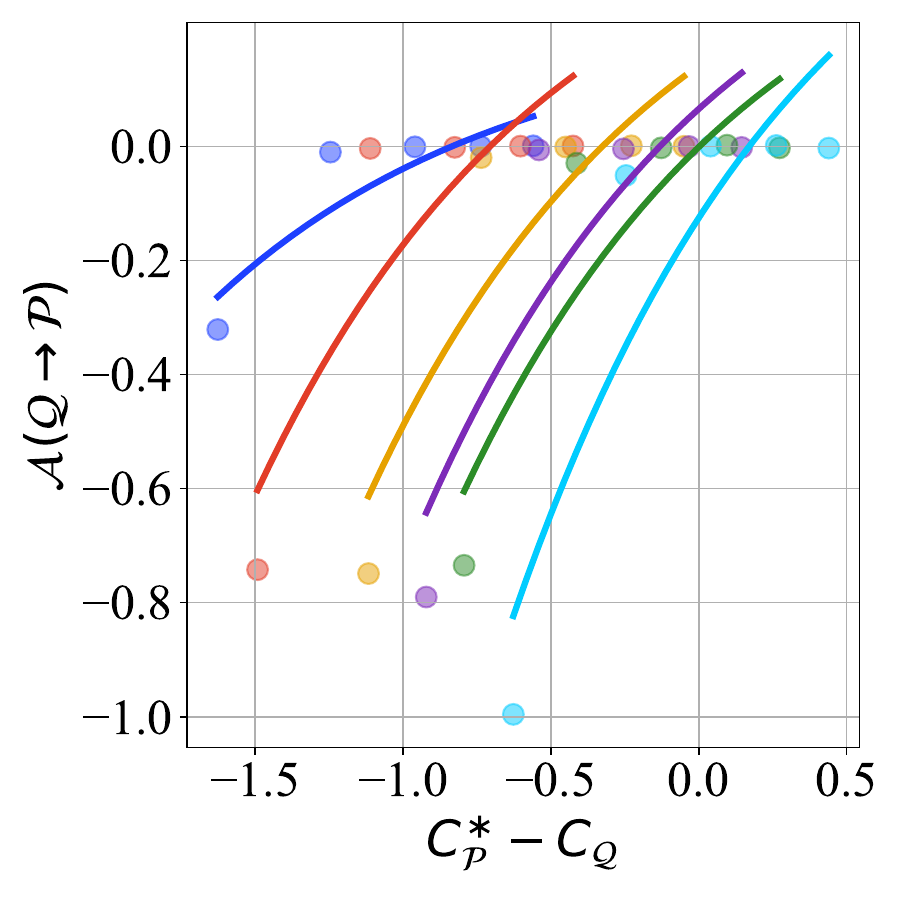}
		\caption{Mistral Nemo 12B}
		\label{app:fig:mistral:12b}
	\end{subfigure}
	\vspace{-1mm}
	\caption{
		The compression order advantage $\coa{\quantonly \rightarrow \pruneonly}$ increases monotonically with the CER difference $\compeqrateprune - \compratequant$ also for Mistral herd models.
		See Appendix~\ref{app:subsec:llm_results} for details.
	}
	\label{app:fig:mistral}
\end{figure}

\subsection{Analysis on Encoder-based Models}
\label{app:subsec:encoder_based}
%
%
\begin{wrapfigure}[10]{R}{4.8cm}
	\vspace{-14mm}
	\centering
	\includegraphics[width=0.9\linewidth]{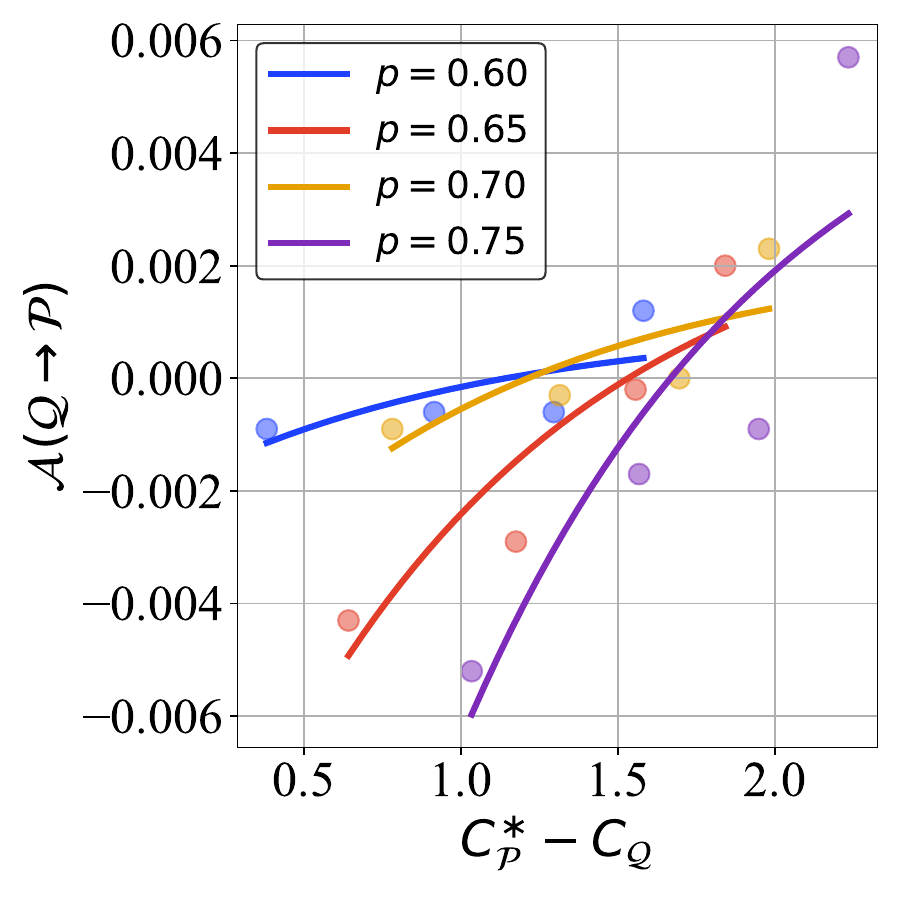}
	\vspace{-4mm}
	\caption{
		The hypothesis holds for encoder-based language models.
		See Appendix~\ref{app:subsec:encoder_based} for details.
	}
	\label{app:fig:encoder_based}
\end{wrapfigure}
%
%
Beyond decoder-only LLMs, we extend our analysis to encoder-based language models to validate the generality of our hypothesis.
Figure~\ref{app:fig:encoder_based} presents the performance of a BERT~\citep{BERT} model under K-prune~\citep{KPrune} ($\prune{\cdot}$) and UniQuanF~\citep{UniQuanF} ($\quant{\cdot}$).
We adopt Spearman correlation as the performance metric, measured on STS-B dataset~\citep{STS-B} from GLUE~\citep{GLUE} benchmark.
We observe a monotonic increase along both axes of $\coa{\quantonly \rightarrow \pruneonly}$ and $\compeqrateprune - \compratequant$, confirming that our hypothesis holds.
%


\begin{figure}[htbp]
	\centering
	\includegraphics[width=0.9\linewidth]{main_result_legend}\\
	\vspace{5mm}
	\hspace{0.01\linewidth}
	\begin{subfigure}[t]{0.32\linewidth}
		\centering
		\includegraphics[width=\linewidth]{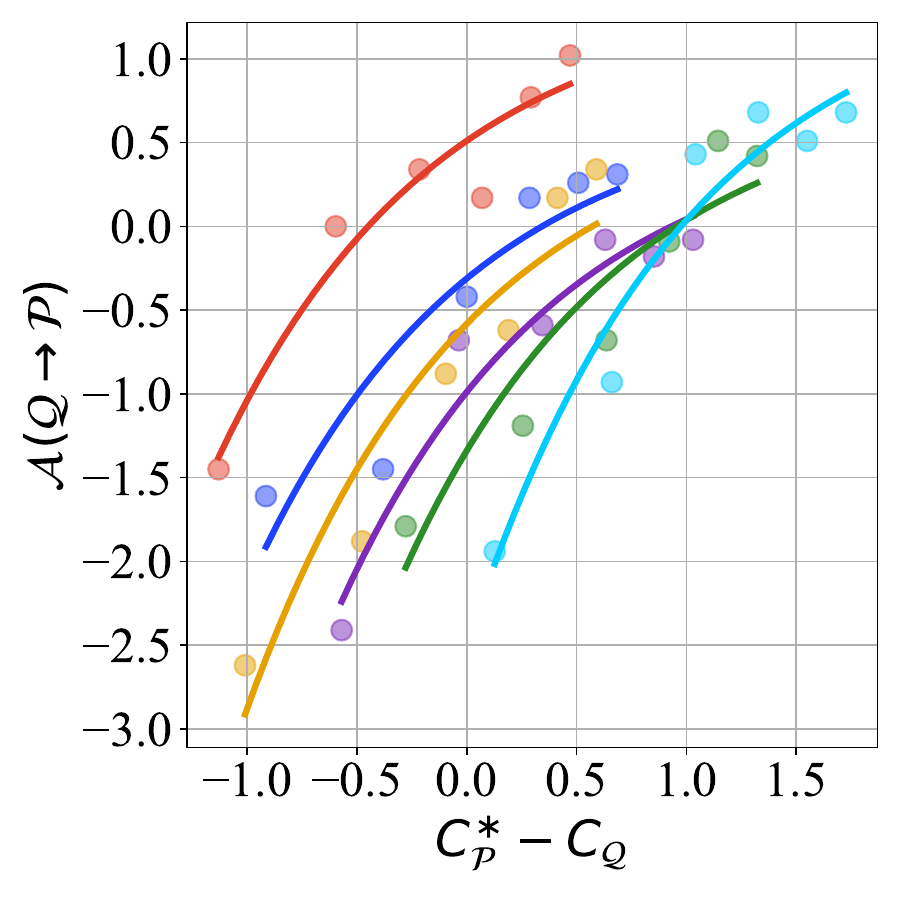}
		\caption{ARC}
		\label{app:fig:csr:arc_challenge}
	\end{subfigure}
	\hspace{15mm}%
	\begin{subfigure}[t]{0.32\linewidth}
		\centering
		\includegraphics[width=\linewidth]{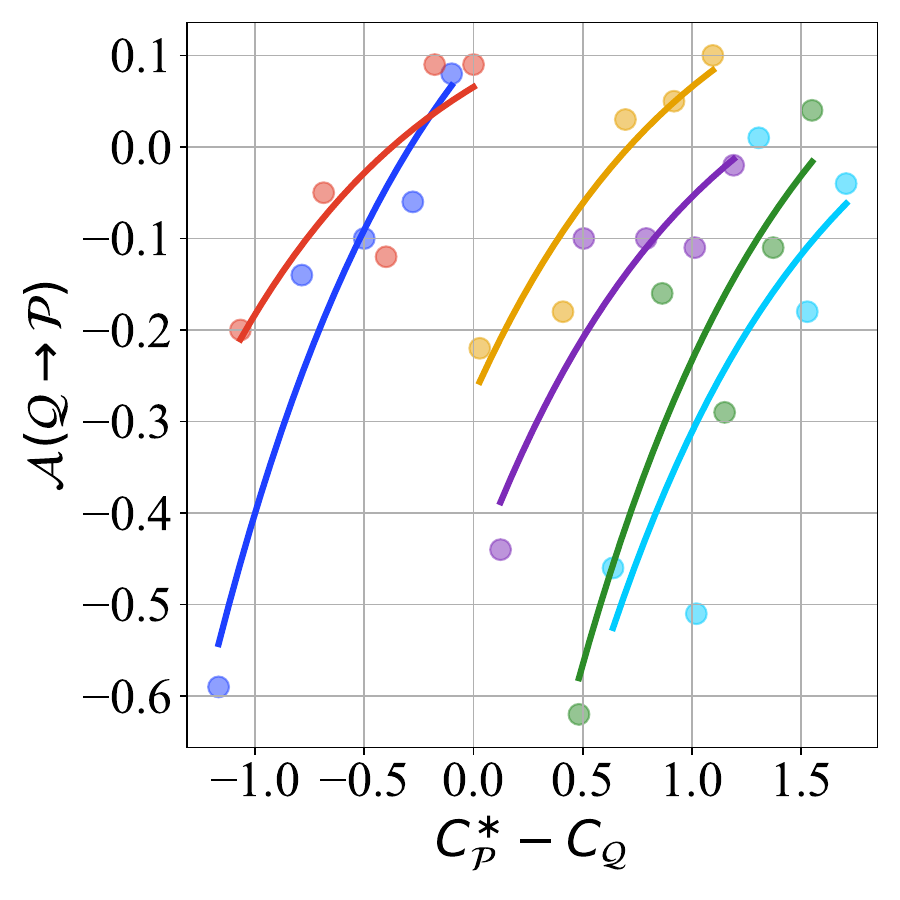}
		\caption{HellaSwag}
		\label{app:fig:csr:hellaswag}
		\vspace{5mm}
	\end{subfigure}
	%
	%
	\begin{subfigure}[t]{0.32\linewidth}
		\centering
		\includegraphics[width=\linewidth]{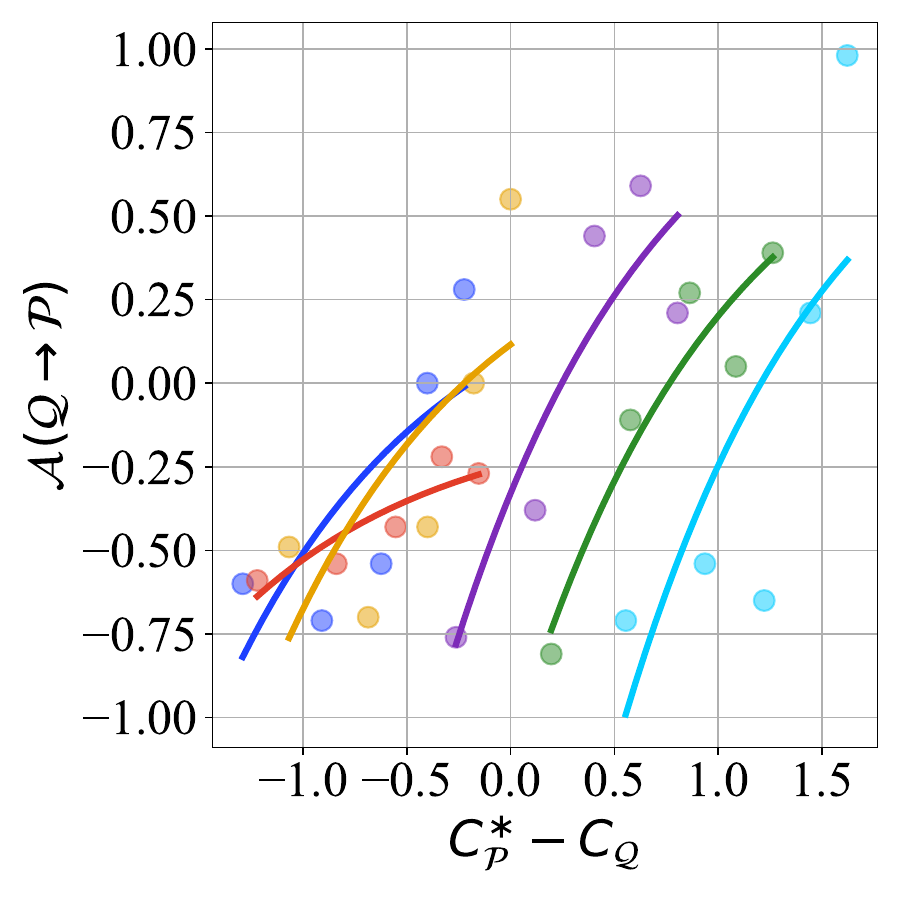}
		\caption{PIQA}
		\label{app:fig:csr:piqa}
	\end{subfigure}
	\hspace{15mm}%
	\begin{subfigure}[t]{0.32\linewidth}
		\centering
		\includegraphics[width=\linewidth]{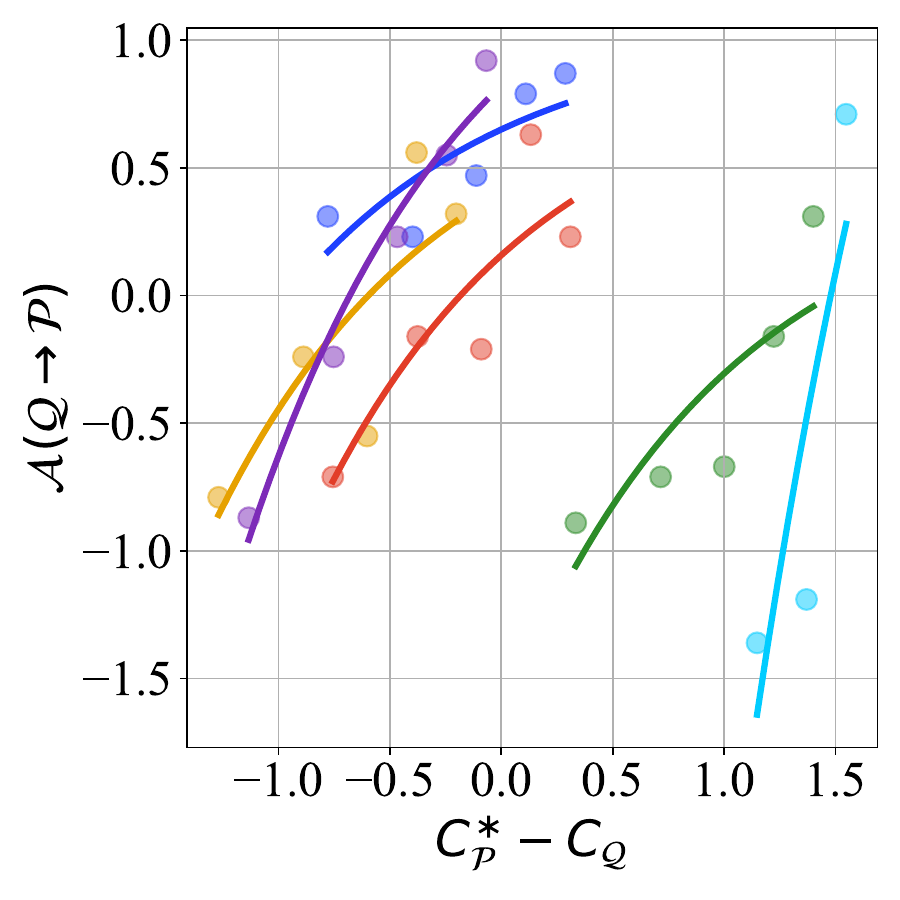}
		\caption{Winogrande}
		\label{app:fig:csr:winogrande}
		\vspace{5mm}
	\end{subfigure}
	%
	%
	\begin{subfigure}[t]{0.32\linewidth}
		\centering
		\includegraphics[width=\linewidth]{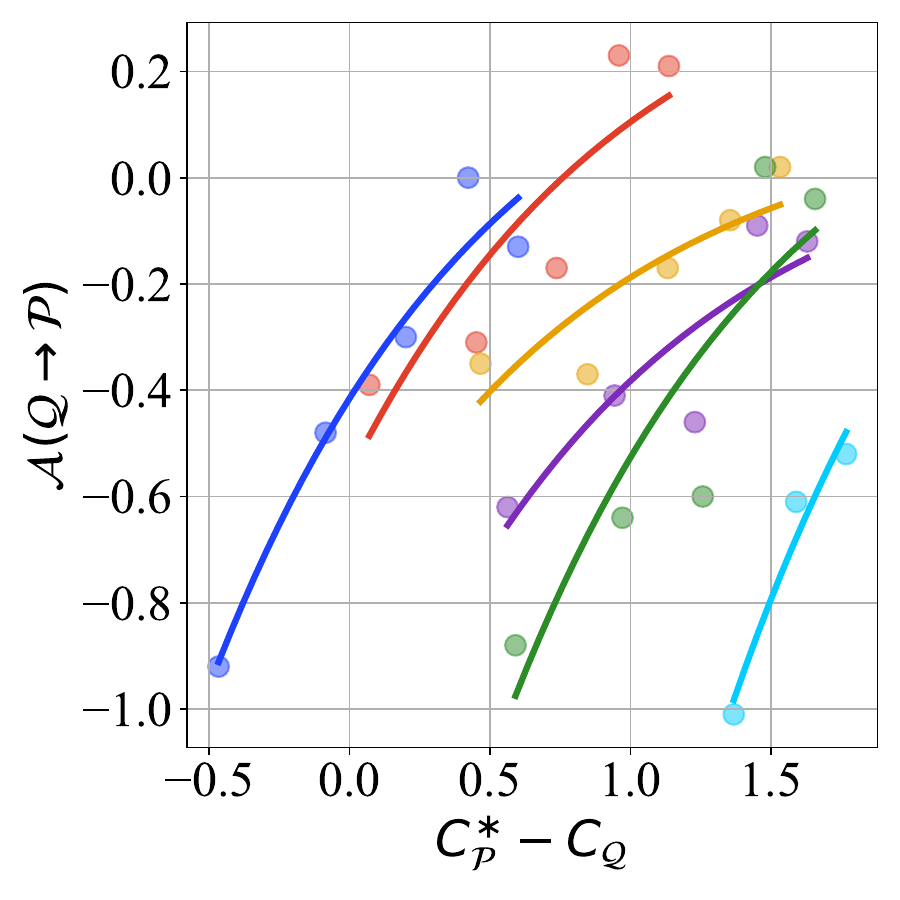}
		\caption{LAMBADA}
		\label{app:fig:csr:lambada}
	\end{subfigure}
	\hspace{15mm}%
	\begin{subfigure}[t]{0.32\linewidth}
		\centering
		\includegraphics[width=\linewidth]{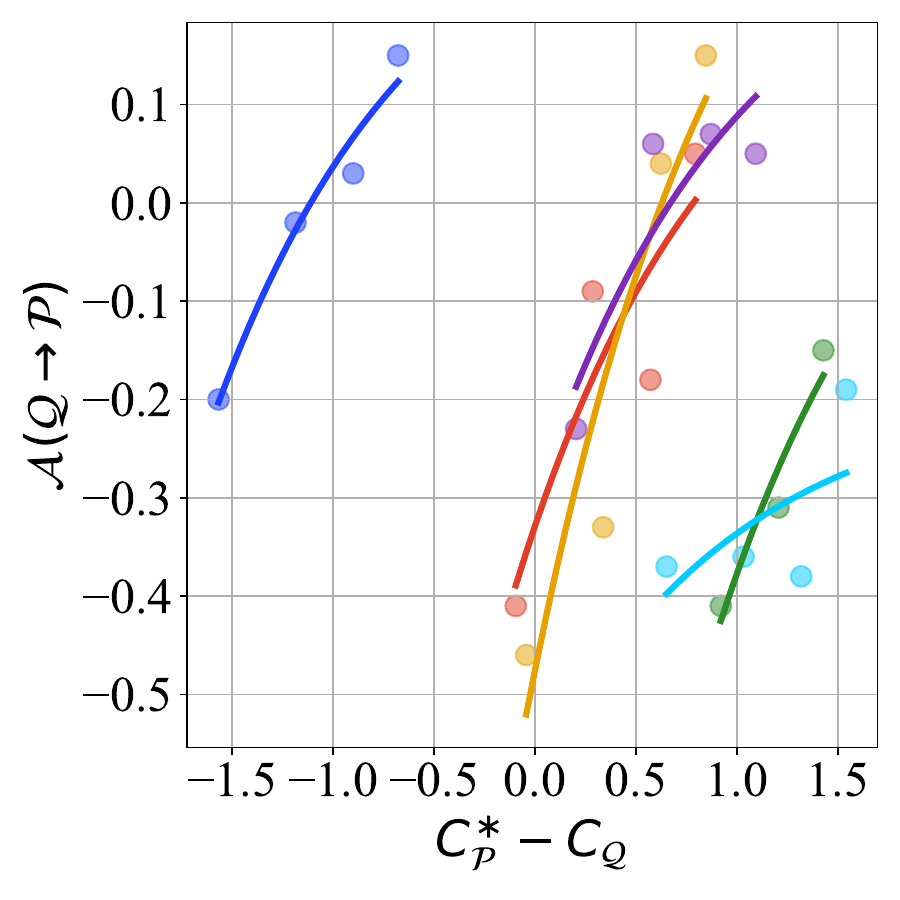}
		\caption{Average}
		\label{app:fig:csr:average}
		\vspace{5mm}
	\end{subfigure}
	%
	%
	\caption{
		Commonsense reasoning task performance of a LLaMA 3 8B model for SparseGPT and QuaRot.
		See Appendix~\ref{app:subsec:csr} for details.
	}
	\label{app:fig:csr}
\end{figure}

\subsection{Commonsense Reasoning Performance}
\label{app:subsec:csr}
Although the negative of perplexity serves as an intuitive and efficient metric $\metric{\cdot}$ for evaluating language models, prior studies~\citep{BeyondPPL, LongPPL} suggest it does not always correlate with real-world performance.
We thus investigate the performance of LLaMA 3 8B model across five commonsense reasoning tasks in Figure~\ref{app:fig:csr}, when applying SparseGPT ($\prune{\cdot}$) and QuaRot ($\quant{\cdot}$).
Results affirm the generality and metric-agnostic nature of our framework, as the hypothesis holds across these tasks.


\begin{table}[t]
	\centering
	\setlength{\tabcolsep}{6.2pt}
	\caption{WikiText-2 perplexity comparison of a LLaMA 3 8B model pruned by SLEB~\citep{SLEB} and SparseGPT~\citep{SparseGPT} under varying pruning ratios, with and without rotation following QuaRot~\citep{QuaRot}.
		See Section~\ref{app:subsec:rotation_pruning} for details.}
	\label{app:tab:rotation_pruning}	
	\renewcommand{\arraystretch}{1}
	\begin{tabular}{c c c c c c c}
		\hline
		\toprule
		\multirow{2}[3]{*}{\textbf{Pruning Ratio}} & \multicolumn{3}{c}{\textbf{SLEB}} & \multicolumn{3}{c}{\textbf{SparseGPT}} \\
		\cmidrule(lr){2-4} \cmidrule(lr){5-7}
		& No rotation & Rotation & \textbf{Difference} & No rotation & Rotation & \textbf{Difference} \\
		\midrule
		Original & \multicolumn{6}{c}{6.137} \\
		\midrule
		0.05    & 6.857   & 6.871   & \textbf{0.014 }
		& 6.140 & 6.154   & \textbf{0.014 }  \\
		0.1      & 8.792   & 8.828   & \textbf{0.036} &
		6.159 & 6.205   & \textbf{0.046} \\
		0.15    & 12.603 & 12.615  & \textbf{0.012} &
		6.213 & 6.352   & \textbf{0.139}  \\
		0.2      & 25.289 & 25.295 & \textbf{0.006} &
		6.330 & 6.629   & \textbf{0.299} \\
		0.25    & 51.212 & 51.560 & \textbf{0.348} &
		6.546 & 7.250   & \textbf{0.704}  \\
		0.3      & 61.502 & 61.901 & \textbf{0.399} &
		6.894 & 8.504   & \textbf{1.610}  \\
		0.35    & 65.997 & 66.234 & \textbf{0.237} &
		7.474 & 20.842 & \textbf{13.368} \\
		0.4      & 92.848 & 93.260 & \textbf{0.412} &
		8.477 & 98.213 & \textbf{89.736} \\
		\bottomrule
		\hline
	\end{tabular}
	\vspace{2mm}
\end{table}


\subsection{Impact of Rotation on Pruning Methods}
\label{app:subsec:rotation_pruning}

In Figure~\ref{fig:rotation_pruning} and Finding 3, we observe that applying rotation without quantization may lead to notable degradation on pruning performance.
To further analyze this, Table~\ref{app:tab:rotation_pruning} compares the performance of LLaMA 3 8B model pruned with and without QuaRot-based rotation, across two pruning methods with different granularities.
We have two observations from the result.
First, rotation-induced degradation scales with the pruning ratio.
This is because higher pruning ratios result in more units being pruned that are altered by rotation, thereby increasing the error.
Second, unstructured pruning (SparseGPT) exhibits significantly higher error compared to structured pruning (SLEB).
This trend is especially evident under high pruning ratios.

\begin{figure}[t]
	\centering
	\includegraphics[width=\linewidth]{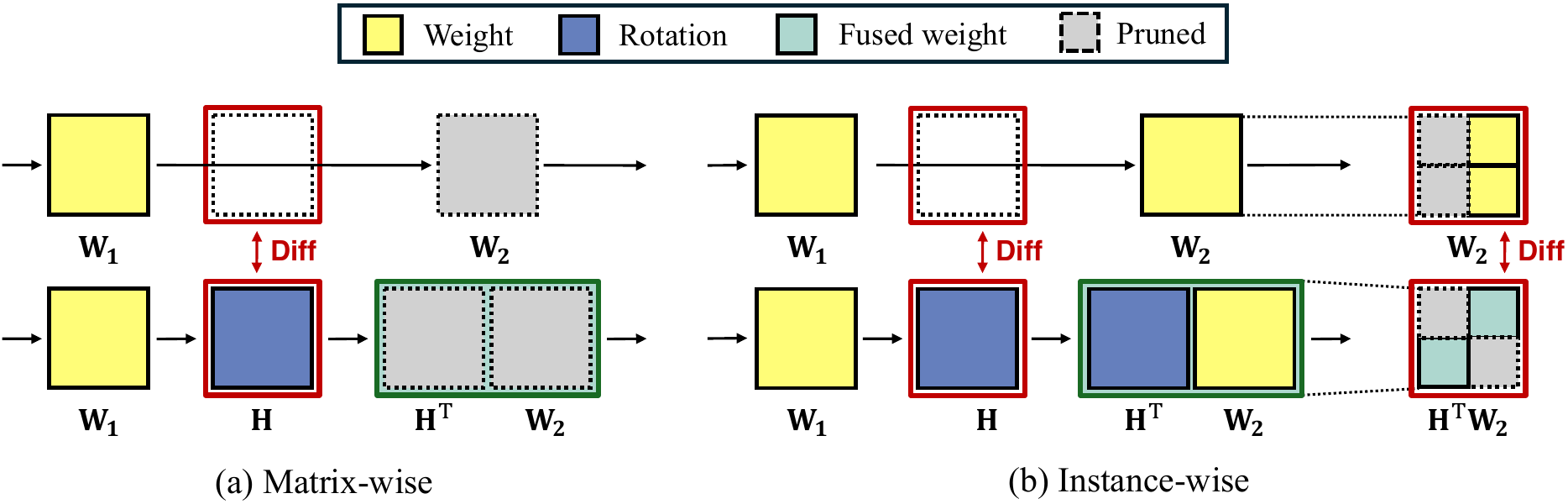}
	\caption{
		\label{app:fig:rotate_prune}
		Two cases of errors when pruning rotated units.
		See Section~\ref{app:subsec:rotation_pruning} for details.
	}
\end{figure}

We therefore investigate the underlying reason behind this phenomenon.
We identify two types of errors induced by pruning, depending on its granularity: matrix-wise and element-wise errors.
Figure~\ref{app:fig:rotate_prune} conceptually illustrates these two cases.

First, in the case of matrix-wise pruning, ignoring rotation during pruning leaves the rotation-induced matrix $\mat{H}$ intact, introducing extra computation and numerical errors compared to the non-rotated case.
As suggested in QuaRot~\citep{QuaRot}, the rotation inverse is fused into the target layer, while the original transform is merged into the preceding normalization layer, leaving un-removed components that generate error during na\"ive pruning.
This type of error scales proportionally with the pruning ratio, as each pruned matrix introduces one such error.

Second, in element-wise pruning, additional errors arise due to rotation-induced changes in unit selection, on top of the matrix-wise error.
As the goal of rotation is to facilitate quantization by flattening activation outliers, multiplying its inverse results in an error compared to the original matrix.
Consequently, the discrepancy in layer content leads to different pruning decisions.
Furthermore, this selection-based error grows with higher pruning ratios due to a greater number of pruned units.

In summary, given these errors, it is crucial to develop pruning techniques that align with rotation-based quantization strategies.

\subsection{A Direct Comparison with Prior Studies}
\label{app:subsec:direct_comparison_interplay}
We discuss how our approach differs from prior studies, particularly SmoothQuant~\citep{SmoothQuant} and~\citet{Interplay}.
Our contribution lies in establishing a general analysis for understanding compression order across diverse methods, whereas these prior works either focus on single-method optimization or analyze specific method pairs under restrictive assumptions.

\smallsection{SmoothQuant~\citep{SmoothQuant}}
SmoothQuant addresses a fundamentally different problem than joint compression order optimization.
Specifically, SmoothQuant optimizes a \textit{single} compression technique (quantization) by mitigating activation outliers through per-channel scaling transformations.
While the paper states that ``SmoothQuant is orthogonal to quantization schemes,'' this refers to its compatibility as a
pre-processing step that can be applied before various quantization methods.
However, SmoothQuant does not examine the order-dependent interaction problem when combining quantization with other compression families such as pruning.
In contrast, our work focuses on understanding how compression order affects the model performance when sequentially combining methods from different compression families.
SmoothQuant may serve as a component within our quantization baselines (i.e., as a pre-processing step before quantization), but our analysis operates at a higher level—determining the optimal ordering between different model compression techniques regardless of the specific quantization implementation.
This distinction is critical: SmoothQuant addresses \textit{intra-method} optimization (improving quantization itself), while we address \textit{inter-method} composition
(ordering across different compression types).

\smallsection{\citet{Interplay}}
As noted in Section~\ref{sec:prelim}, only a few studies have addressed how the order of compression methods affects the model performance.
Among them, \citet{Interplay} stands out as the only study that attempts a theoretical approach to the problem.
They examine the interaction between pruning and quantization, showing that the two are not orthogonal as assumed by previous works.
Furthermore, they argue that pruning followed by quantization is universally optimal.

However, their framework suffers from three significant limitations.
First, their framework relies on oversimplified assumptions that hinder practical applicability.
Specifically, they focus solely on magnitude-based pruning (removes weights based on their absolute values) and max-scaled block-wise quantization (uniformly rescaling blocks using their maximum value), both of which are na\"ive approaches that are less practical and often fail to preserve accuracy.
Second, their analysis is confined to a minimal set of scenarios, failing to address diverse architectures or methods.
Beyond the limited set of methods, their experiments also consider only the combination of two techniques—pruning and quantization—on decoder-based LLMs, lacking broader coverage of models and compression approaches.
Lastly, the framework cannot be generalized across different settings, as many counterexamples have shown that pruning-before-quantization is not always optimal.
Motivated by these gaps, we aim for a more general formulation that holds across methods, models, and metrics, thereby introducing the Progressive Intensity Hypothesis.

\subsection{Violation Cases of the Hypothesis}
\label{app:subsec:violation}
Although our hypothesis is highly general and robust, we still observe cases where it does not hold.
These cases largely fall into three categories: severe performance collapse, full model re-training, and increase of order-affected layers.
First, each model exhibits a different tolerance to compression, with performance dropping exponentially beyond a certain ratio.
While these settings are impractical due to severe performance loss, we observe cases where applying the stronger method first performs better.
This may be because the error is already too large, violating our assumption of well-designed compression in Section~\ref{sec:theory}; applying the stronger method first might help reduce the total error.
For less compression-robust models like decoder-based LLMs, we observe earlier breakdowns—such as diminishing advantage when pruning ratio increases at fixed bit-width (Figure~\ref{fig:llm_result:llama2_7B}).
Second, when strong full-training is applied, the advantage from compression order may invert.
Compression order serves merely as initialization, and the retraining process dominates, making it difficult to attribute outcomes to order alone.
We plan to investigate these and potentially other exceptions more rigorously in future work.
Lastly, in practical situations, increasing $\comprateprune$ may result in increase of order-affected layers, leading to a violation of the hypothesis.
We analyze the details in the following paragraph.

\begin{figure}[t]
	\centering
	\includegraphics[width=\linewidth]{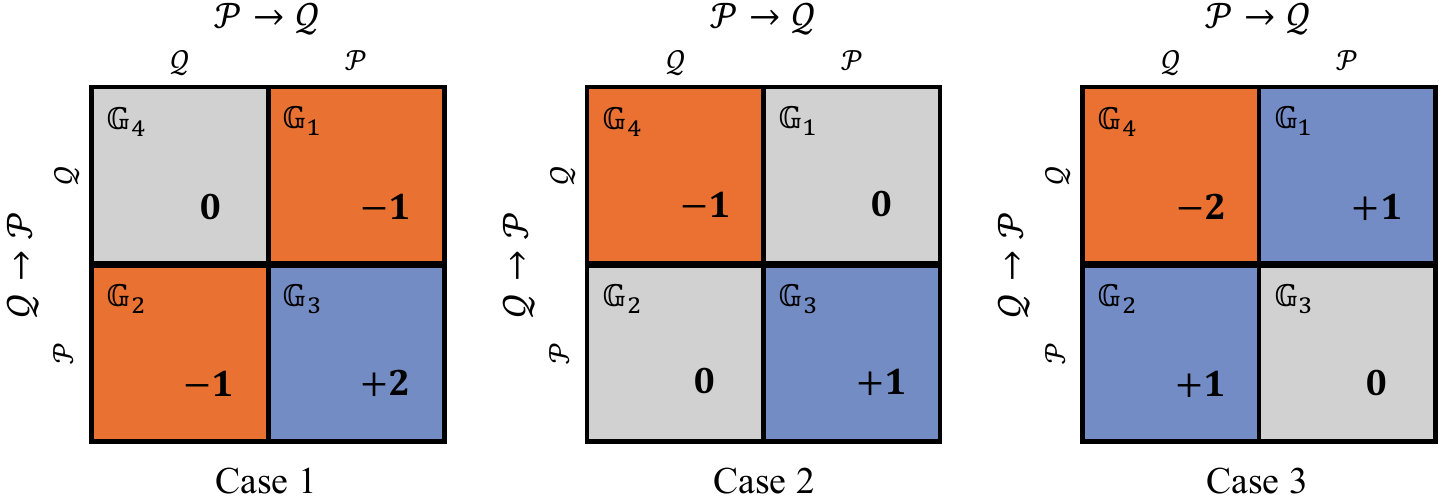}
	\caption{
		If the number of pruned units increases by one, the change in unit allocation across groups fall into three distinct cases.
		See Appendix~\ref{app:subsec:violation} for details.
	}
	\label{app:fig:violation}
\end{figure}

\smallsection{Impact of $\comprateprune$}
Without loss of generality, we consider the case when only $\comprateprune$ increases.
Increasing $\comprateprune$ implies a stronger pruning effect, since it lowers $\metric{\prune{\model}}$, resulting in a decrease in $\pg{\pruneonly, \quantonly}$ and a corresponding increase in $\compeqrateprune$.
Hence, to ensure monotonicity and satisfy Hypothesis~\ref{hypothesis}, $\coa{\quantonly \rightarrow \pruneonly}$ should increase accordingly.
Note that $\compratequant$ is fixed while analyzing the effect of $\comprateprune$.

To analyze the effect of increasing $\comprateprune$, we first consider a local step in which the total number of pruned units increases by one.
For the initial pruning ratio $p$ and the increased ratio $p'$ under the same granularity $\typegran{\pruneonly}$, the following relation holds:
\[
p' \cdot |\units{\model; \typegran{\pruneonly}}|
= p \cdot |\units{\model; \typegran{\pruneonly}}| + 1.
\]
From the definition of compression ratio $\comprateprune = 1/(1 - p)$, larger pruning ratios correspond to larger compression ratios.
By repeating this incremental process, we can construct any pruning ratio.
Under disjoint selectivity, each unit is exclusively assigned to one compression method, allowing us to partition the units into four disjoint groups as discussed in Appendix~\ref{app:subsec:proof_thm1}.

Increasing the pruning ratio from $p$ to $p'$ result in three possible changes in the group configuration: the number of affected layers by order (i.e., $|\groupone|$ or $|\grouptwo|$; note that two values are equal.) 1) decreases by one, 2) remains unchanged, or 3) increases by one.
The three cases are visualized in Figure~\ref{app:fig:violation}.
From Theorem~\ref{thm:1}, $\coa{\quantonly \rightarrow \pruneonly}
= \beta \cdot \big(\sum_{\layeri \in \grouptwo} {g(\layeri)} - \sum_{\layeri \in \groupone} {g(\layeri)} \big)$ (where $g(\layeri) = \|\errorq{\weighti, \acti} \|_F^2 - \| -\weighti \acti \|_F^2$) should be preserved or increased to satisfy Hypothesis~\ref{hypothesis}.
However, this condition is fulfilled in only Cases 1 and 2, but not in Case 3.

\begin{itemize}[leftmargin=3.4mm, itemsep=0mm]
	\item \textbf{Case 1: Number of layers affected by order decreases by one.}
	If a layer is no longer affected by compression due to order change, then another layer must also be excluded to preserve the total number of pruned layers which should be increased by one.
	Thus, $|\groupone|$ and $|\grouptwo|$ each decrease by one, while $|\groupthree|$ increases by two.
	Similar to Case 1 of Appendix~\ref{app:subsec:proof_thm2}, as the increase in $\comprateprune$ eliminates a negative loss term, $\coa{\quantonly \rightarrow \pruneonly}$ increases.
	\item \textbf{Case 2: Number of layers affected by order remains unchanged.}
	If the additionally pruned layer is always pruned regardless of the compression order, then $|\groupthree|$ increases by 1 while $|\groupfour|$ decreases by 1.
	Similar to Case 2 of Appendix~\ref{app:subsec:proof_thm2}, as the loss-contributing groups $\groupone$ and $\grouptwo$ do not change, the compression order advantage $\coa{\quantonly \rightarrow \pruneonly}$ remains unaffected.
	\item \textbf{Case 3: Number of layers affected by order increases by one.}
	The aforementioned two cases could not increase the number of order-dependent units, whereas there should exist a case where both groups $|\groupone|$ and $|\grouptwo|$ increases by one.
	As the number of total pruned layers should be increased by one, the added layers should be originated from $\groupfour$, i.e., $|\groupfour|$ decreases by two.
	In contrast to Case 1, $\coa{\quantonly \rightarrow \pruneonly}$ may decrease due to the increase of negative loss terms.
\end{itemize}

As the conditions under which each case emerges differ across specific configurations, this phenomenon is not analyzed in general settings.
We leave a precise characterization of the conditions under which increasing the pruning ratio may invalidate the progressive intensity hypothesis as an important direction for future work.

\subsection{Additional Remarks}
\label{app:subsec:remarks}

\smallsection{Limitations of Current Work}
We introduce a broadly applicable hypothesis that can be extended to diverse compression methods and model types across different domains.
Still, we acknowledge three important limitations in our current work.
First, due to the general nature of our framework, it does not provide detailed analysis for each specific combination of methods.
While our hypothesis captures high-level trends, it does not define the best compression sequence for individual cases.
This motivates research into discovering the best compression orderings under practical scenarios.
Second, our study is limited to joint model compression in plug-and-play settings where methods are combined post-hoc.
As demands for higher compression grow, integrated design strategies should be investigated beyond simple combinations.
Lastly, our framework does not yet provide explicit predictive rules or precise estimation of how much better one compression order is than another.
Capturing the nonlinear and cross-layer effects required for precise sign or value prediction of compression-order advantage $\mathcal{A}(\cdot)$ remains an open problem.
We therefore consider the development of predictive models for $\mathcal{A}$ and meta-learning approaches for automatic order selection as a promising direction for future research.

\smallsection{Future Work}
In addition to addressing the aforementioned limitations, future directions may also include extensions of our current framework.
First, a systematic study of interference across different pipeline designs would provide deeper insights beyond our current empirical findings.
Another direction is to automate compression order selection based on observed trends.
A unified approach that generalizes across cases may offer a better understanding on the role of compression order.
Lastly, evaluating our hypothesis on emerging architectures such as Mixture-of-Experts and multimodal LLMs may broaden its generality.

\smallsection{Usage of AI Assistants}
We employ ChatGPT\footnote{\url{https://chatgpt.com/}}
(GPT-4o) and Perplexity\footnote{\url{https://www.perplexity.ai/}}
exclusively for language polishing purposes; for improving grammar and clarity at the sentence level.
We do not use them for any research-related tasks, including code implementation, theoretical derivation, and result analysis.


\end{document}